\documentclass{article}

\PassOptionsToPackage{numbers, sort&compress}{natbib}

 \usepackage[main, final]{neurips_2026}

\usepackage[utf8]{inputenc} 
\usepackage[T1]{fontenc}    
\usepackage{hyperref}       
\usepackage{url}            
\usepackage{booktabs}       
\usepackage{amsfonts}       
\usepackage{nicefrac}       
\usepackage{microtype}      
\usepackage{xcolor}         

\usepackage{graphicx}
\definecolor{lightcitegreen}{RGB}{76, 175, 80}
\definecolor{lightciteblue}{RGB}{65,105,225}
\hypersetup{
    colorlinks=true,
    citecolor=lightciteblue,
    linkcolor=lightciteblue
}

\hypersetup{
    colorlinks=true,
    citecolor=lightciteblue,
    linkcolor=lightciteblue,
    urlcolor=lightciteblue
}

\usepackage{amsmath}

\usepackage{amsmath,amsfonts,bm}

\def\eqref#1{eq.~\ref{#1}}

\def\1{\bm{1}}

\def\vb{{\bm{b}}}

\def\mA{{\bm{A}}}

\def\mI{{\bm{I}}}

\def\mL{{\bm{L}}}

\def\mP{{\bm{P}}}

\def\mR{{\bm{R}}}
\def\mS{{\bm{S}}}

\def\mU{{\bm{U}}}
\def\mV{{\bm{V}}}
\def\mW{{\bm{W}}}
\def\mX{{\bm{X}}}
\def\mY{{\bm{Y}}}

\def\mSigma{{\bm{\Sigma}}}

\DeclareMathAlphabet{\mathsfit}{\encodingdefault}{\sfdefault}{m}{sl}
\SetMathAlphabet{\mathsfit}{bold}{\encodingdefault}{\sfdefault}{bx}{n}

\usepackage{mathtools}
\usepackage{wrapfig}
\usepackage{multirow}
\usepackage{graphicx}

\usepackage{subcaption}
\usepackage{amsthm}
\newtheorem{proposition}{Proposition}
\usepackage{algorithm}
\usepackage{algpseudocode}

\title{DyRA: Dynamic Residual Approximation for \\ Efficient Matrix Multiplication in DNNs}

\author{
Daewon Chae$^{1}$ \quad
Hyunwon Chung$^{1}$ \quad
Changwoo Lee$^{2}$\thanks{Contributed to this work while at the University of Michigan.} \quad
Hun-Seok Kim$^{1}$ \\
$^{1}$Department of EECS, University of Michigan \quad
$^{2}$Qualcomm AI Research \\
\texttt{\{daewon, hyunwon, hunseok\}@umich.edu} \quad
\texttt{changle@qti.qualcomm.com}
}

\begin{document}

\maketitle
\setcounter{footnote}{0}

\vspace{-1em}
\begin{abstract}
Large-scale foundation models achieve strong performance across diverse tasks, but their size makes inference costly, largely due to dense matrix multiplications. Prior work reduces this cost by replacing dense weight matrices with efficient structured forms such as low-rank factorizations. However, these methods approximate weights rather than the output activations that determine inference accuracy. Consequently, small weight-space errors can be amplified by input activations, producing large output errors.
In this work, we propose DyRA, an input-adaptive method that improves structured matrix multiplication approximation by correcting residual output errors during inference. We show that matrix multiplication can be approximated more effectively by directly optimizing low-rank factors of the output. DyRA builds on this insight by dynamically approximating and correcting the output error introduced by structured weight approximations. This combines efficient structured computation with input-dependent correction, yielding a more faithful approximation of full matrix multiplication under the same computational budget.
Across vision, speech, and language models, DyRA consistently improves the accuracy–efficiency trade-off over structured weight approximations alone. Notably, DyRA achieves a $1.5\times$ end-to-end GPU speedup for DINOv3 while reducing accuracy degradation by more than $3\times$ relative to weight-only baselines. Our code is available at \href{https://github.com/daewon88/DyRA}{\texttt{https://github.com/daewon88/DyRA}}.
\end{abstract}

\section{Introduction}
Recent advances in deep learning have achieved state-of-the-art performance across diverse tasks, but at substantial computational cost. Dense matrix multiplications account for much of the computation in modern deep networks\footnote{For example, in ViT-L/16~\cite{vit}, linear layers account for approximately 96.7\% of total FLOPs.}, motivating work on accelerating them by replacing dense weight matrices with efficient \textit{structured} forms~\cite{blast,monarch,gblr,btt,butterfly,mode}. Low-rank approximation is a representative example, while block-structured variants such as Monarch~\cite{monarch} and BLAST~\cite{blast} offer more expressive alternatives. Given a pre-trained model, these methods approximate its weight matrices via closed-form decompositions or iterative optimization, enabling faster inference with reduced FLOPs.

Despite their effectiveness, these methods are fundamentally weight-centric: they optimize approximations in weight space without explicitly accounting for the input activations seen at inference time. However, the actual computation is determined by the realized product between activations and weights, not by the weights alone. Thus, even small weight-space residuals can be amplified by the input, causing large output activation errors~\cite{asvd}. We illustrate this limitation on DINOv3~\cite{dinov3}, a large-scale vision foundation model. As shown in Figure~\ref{fig:dino_teaser}, replacing all dense linear layers with structured matrices such as low-rank, Monarch, and BLAST makes PCA feature visualizations noticeably noisier and less semantically organized than those of the original model. This demonstrates that approximating weights alone may fail to preserve model behavior.

To address this limitation, we propose DyRA (\textbf{Dy}namic \textbf{R}esidual \textbf{A}pproximation), an input-adaptive method that improves structured matrix multiplication by \textit{correcting output errors on the fly}. DyRA augments a static weight approximation with a dynamic correction adapted to the input activations. This correction is computed through nested optimization during inference to minimize residual output errors induced by the approximated weight matrix. Rather than evaluating the exact residual through an expensive dense multiplication, DyRA uses low-rank factors to approximate the residual directly in the output space, not the weight space. These factors are efficiently computed at inference time using alternating least squares, without explicitly forming the dense residual product. Consequently, DyRA provides input-adaptive output correction that more faithfully approximates the original matrix multiplication under a fixed computational budget. Figure~\ref{fig:dino_teaser} shows that this correction substantially preserves the semantic structure of DINOv3 features compared with weight-only approximations.

We validate DyRA as a test-time model approximation and acceleration technique across pre-trained vision, language, and speech models. Combined with structured matrices such as low-rank, Monarch, and BLAST, DyRA consistently improves the accuracy--efficiency trade-off over weight-only baselines. For example, on semantic segmentation with DINOv3 at a 60\% FLOPs reduction, DyRA limits the mIoU drop to 1.7 points, whereas the weight-only baseline drops by 54.1 points. These gains translate to practical runtime improvements: on DINOv3, it achieves a $1.52\times$ GPU speedup while incurring only 1/3 of the accuracy drop of the weight-only baseline at the same speedup.

\begin{figure}[!t]
    \begin{center}
    \includegraphics[width=1.00\linewidth]{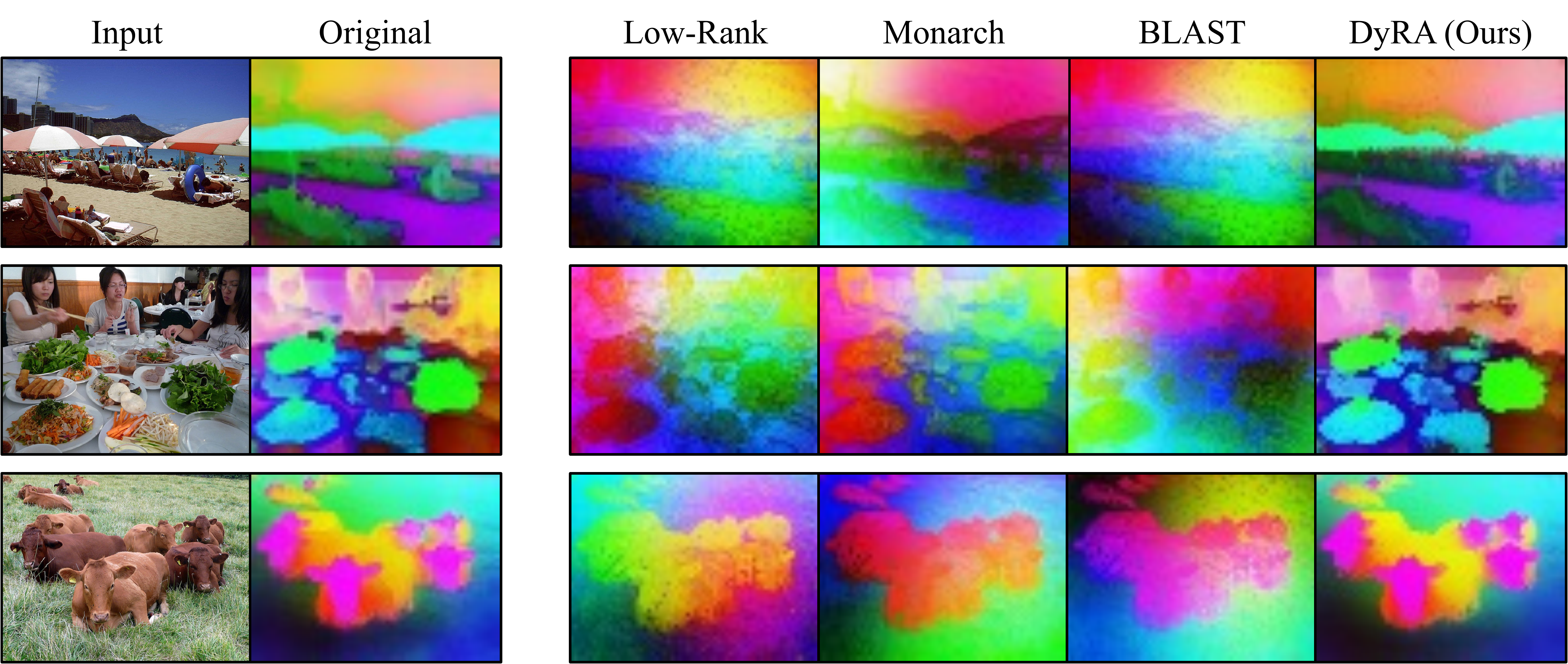}
    \end{center}
    \caption{DINOv3 feature visualizations under a 60\% FLOPs reduction. After replacing all dense linear layers with structured weights, we visualize the resulting features using the DINO PCA protocol, mapping the first three principal components to RGB.  Weight-only approximations produce noisy, semantically degraded features, while DyRA preserves structures close to the original model.}
    \vspace{-1em}
    \label{fig:dino_teaser}
\end{figure}


\section{Proposed Method}
\subsection{Overview}
Starting from a pre-trained weight matrix $\mW$, matrix approximation or model compression methods approximate the weight with a computationally efficient structured form $\widehat{\mW}$ (e.g., low-rank or block-structured approximation). The original weight matrix can be written as
\begin{equation}
\mW = \widehat{\mW} + \mW_{\mathrm{res}},    
\end{equation}
where $\mW_{\mathrm{res}} = \mW - \widehat{\mW}$ denotes the residual. Accordingly, the exact matrix product can be written as
\begin{equation}
    \mX\mW = \mX\widehat{\mW} + \mX\mW_{\mathrm{res}}
\end{equation}
given input activation $\mX$.
Conventional approximation methods use only the first term and approximate the full product as $\mX\mW \approx \mX\widehat{\mW}$.
The approximation error is thus given by $\mX\mW_{\mathrm{res}}$, rather than the residual weight matrix $\mW_{\mathrm{res}}$ alone. This implies that, even when $\widehat{\mW}$ is chosen to minimize $\|\mW_{\mathrm{res}}\|$, the resulting error remains input-dependent and may still be large in the output space due to its interaction with the input $\mX$.

To address this limitation, DyRA explicitly accounts for the input-dependent residual product $\mX\mW_{\mathrm{res}}$ by approximating it at inference time. This augments the static approximation $\mX\widehat{\mW}$ with a dynamic correction, yielding a more faithful estimate of $\mX\mW$ within a fixed computational budget. In the following, we introduce the \textit{dynamic residual approximation} that enables this correction efficiently, and analyze its key properties.

\subsection{Dynamic Residual Approximation}
Given an input matrix $\mX \in \mathbb{R}^{n \times c_1}$ and a residual weight matrix
$\mW_{\mathrm{res}} \in \mathbb{R}^{c_1 \times c_2}$, our goal is to efficiently approximate the
residual product $\mX\mW_{\mathrm{res}}$ at inference time \textit{without performing the full matrix multiplication}.
To this end, we formulate the residual output computation as an optimization problem over low-rank factors.
Specifically, we seek a low-rank approximation
$\mX\mW_{\mathrm{res}} \approx \mL\mR$, where
$\mL \in \mathbb{R}^{n \times r}$ and
$\mR \in \mathbb{R}^{r \times c_2}$ with
$r \ll \min(n,c_2)$, by solving
\begin{equation}
\min_{\mL \in \mathbb{R}^{n \times r},\, \mR \in \mathbb{R}^{r \times c_2}}
\left\|\mX\mW_{\mathrm{res}} - \mL\mR\right\|_F^2.
\label{eq:low_rank}
\end{equation}
This objective directly approximates the realized residual output $\mX\mW_{\mathrm{res}}$, rather than the residual weight matrix alone.
Its optimal rank-$r$ solution is given by the truncated singular value decomposition (SVD) of $\mX\mW_{\mathrm{res}}$~\cite{svd}. However, computing this solution requires explicitly forming the full product $\mX\mW_{\mathrm{res}}$, which defeats the purpose of reducing computational cost. We therefore directly optimize the low-rank factors $\mL$ and $\mR$
without explicitly forming $\mX\mW_{\mathrm{res}}$. 
Specifically, we optimize the objective in Eq.~\ref{eq:low_rank} using alternating least squares (ALS).
Starting from an initial right factor $\mR^{(0)}$, ALS alternates between least-squares subproblems obtained by fixing one factor and optimizing the other. This yields the following closed-form updates at iteration $t$:
\footnote{In practice, we use a ridge-regularized version of these updates by
adding a small $\lambda \ge 0$ to the Gram matrices, i.e.,
$\mR^{(t)}\mR^{(t)\top}+\lambda\mI_r$ and
$\mL^{(t+1)\top}\mL^{(t+1)}+\lambda\mI_r$, for numerical stability.}
\begin{align}
    \mL^{(t+1)}
    &=
    \mX \left(\mW_{\mathrm{res}} \mR^{(t)\top}\right)
    \left(\mR^{(t)} \mR^{(t)\top}\right)^{-1},
    \label{eq:L_update} \\
    \mR^{(t+1)}
    &=
    \left(\mL^{(t+1)\top} \mL^{(t+1)}\right)^{-1}
    \left(\mL^{(t+1)\top} \mX\right) \mW_{\mathrm{res}}.
    \label{eq:R_update}
\end{align}
A key efficiency advantage of this ALS formulation is that each update is performed through low-rank ($r$) intermediate computations. In the $\mL$ update, $\mW_{\mathrm{res}}$ is first projected onto the rank-$r$ right factor, yielding
$\mW_{\mathrm{res}}\mR^{(t)\top} \in \mathbb{R}^{c_1 \times r}$.
This reduces the output dimension from $c_2$ to $r$, so that $\mX$ is multiplied only with a slim $c_1 \times r$ matrix. Similarly, in the $\mR$ update, the input matrix is first projected onto the rank-$r$ left factor as
$\mL^{(t+1)\top}\mX \in \mathbb{R}^{r \times c_1}$, reducing the input side from $n$ to $r$ before being multiplied by $\mW_{\mathrm{res}}$. As a result, the residual output is estimated efficiently as $\mL\mR \approx \mX\mW_{\mathrm{res}}$ using low-dimensional intermediate products, without explicitly forming the full residual output.

\begin{figure}[t]
    \begin{center}
    \includegraphics[width=1.0\linewidth]{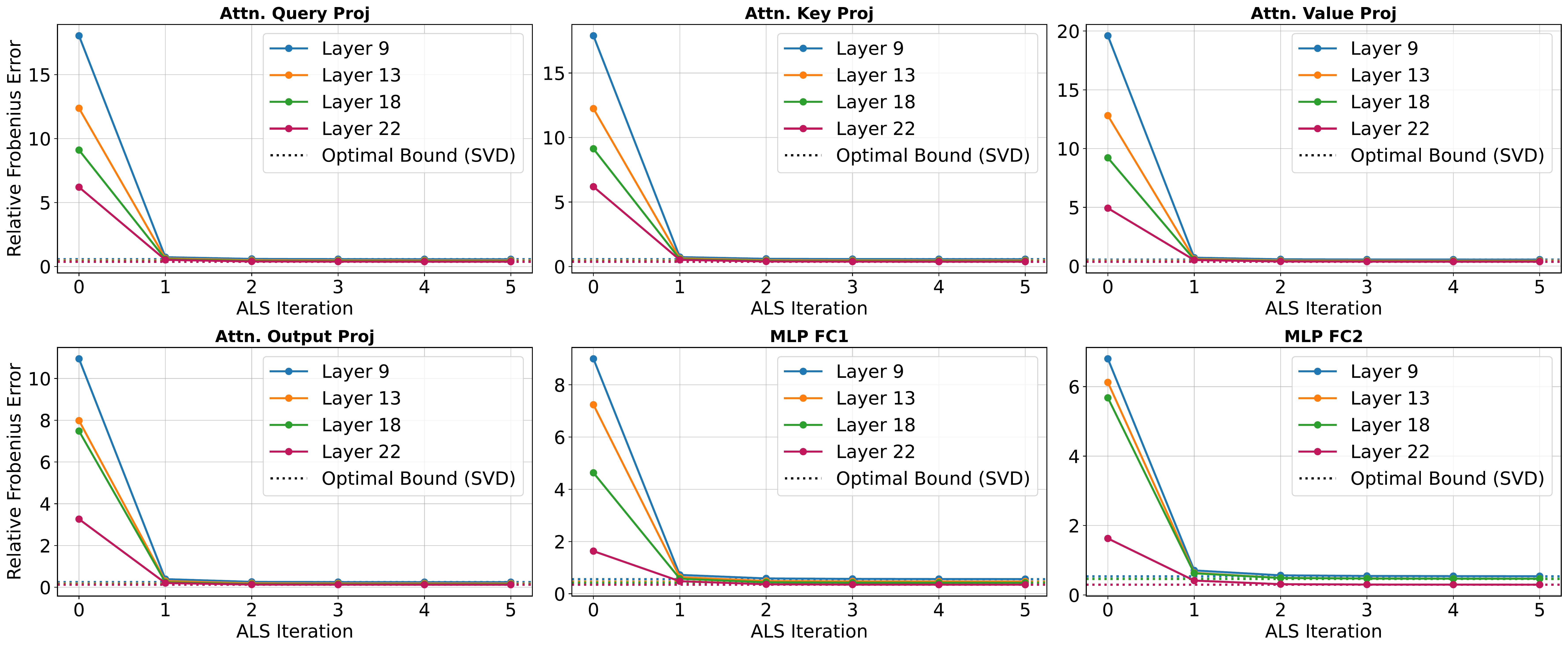}
    \end{center}
    \caption{Convergence of ALS-based residual product approximation across ViT layers. Relative squared Frobenius error is measured for attention projections and MLP layers with target rank $r=16$.}
    \label{fig:als_convergence}
    \vspace{-1em}
\end{figure}

\paragraph{Convergence of ALS Residual Approximation}
\label{paragraph:convergence}
We empirically examine the convergence of the ALS-based residual product approximation using a Vision Transformer (ViT)~\cite{vit} evaluated with ImageNet~\cite{imagenet} samples. We apply ALS to approximate the residual product $\mX\mW_{\mathrm{res}}$ at multiple linear layers of the model with a small target rank $r=16$, initializing $\mR^{(0)}$ from a Gaussian distribution. In Figure~\ref{fig:als_convergence}, we report the relative squared Frobenius error across ALS steps for both attention projections and multi-layer perceptron (MLP) layers.
Across layers, the approximation rapidly approaches the optimal rank-$r$ solution within only one or two iterations.
These results indicate that only a single update per factor is sufficient to effectively approximate the residual product, demonstrating the efficiency of ALS for the dynamic residual approximation. Further details on the experimental setup, along with additional results across various layers, are provided in Appendix~\ref{sup:convergence}.

\subsection{Analysis of DyRA with One-Step ALS Residual Approximation}
Given the rapid convergence observed above, we use only one ALS iteration as the default configuration of DyRA throughout this work.\footnote{
While more ALS iterations can further refine the approximation, one iteration captures most of the gain. Thus, under a fixed computation budget, we use a single iteration and allocate the saved computation to using a larger rank $r$.
}
The resulting approximation is
\begin{equation}
    \mX\mW
    \approx
    \mX\widehat{\mW}+\operatorname{ALS}_{1}\!\left(\mX\mW_{\mathrm{res}}\right)=
    \mX\widehat{\mW}+\mL^{(1)}\mR^{(1)},
    \label{eq:drama_final}
\end{equation}
where $\operatorname{ALS}_{1}(\cdot)$ denotes one ALS iteration, which updates $\mL$ once and $\mR$ once. Notably, the left-factor update admits an efficient implementation:
\begin{equation}
    \mL^{(1)}
    =
    \mX
    \underbrace{
    \left(\mW_{\mathrm{res}}\mR^{(0)\top}\right)
    \left(\mR^{(0)}\mR^{(0)\top}\right)^{-1}
    }_{\mP_{\mL}\;:\;\text{precomputable}}
    \label{eq:init_L_update}
\end{equation}
where $\mP_{\mL}$ is \textit{precomputed offline} as it depends only on $\mW_{\mathrm{res}}$ and the fixed initialization $\mR^{(0)}$. Thus, computing $\mL^{(1)}$ at inference time requires only multiplying $\mX$ by a slim $c_1\times r$ matrix $\mP_{\mL}$, reducing the online cost of $\operatorname{ALS}_1$ largely to the $\mR^{(1)}$ update.

\paragraph{Subspace Comparison under the Same Rank Increase }
\label{paragraph:subspace}
We now analyze this one-step variant to clarify what DyRA contributes beyond static approximation. While DyRA can be combined with any structured weight $\widehat{\mW}$ such as Monarch~\cite{monarch} and BLAST~\cite{blast}, we focus on the low-rank case for analytical clarity. We first compare DyRA with simply increasing the rank of a static low-rank approximation. We then quantify the additional cost of DyRA and use it to construct a FLOPs-matched static baseline.

Specifically, we compare DyRA at correction rank $r$ with a static (weight-only) low-rank approximation whose rank is increased by the same amount, from $k$ to $k+r$. 
For DyRA, we take the rank-$k$ truncated SVD of $\mW$ as the static approximation and apply rank-$r$ DyRA correction. The resulting DyRA is compared with the static rank-$(k+r)$ truncated SVD approximation of the weight.
The following proposition characterizes this difference.
\begin{proposition}
\label{proposition1}
Let $\mW=\mU\mSigma\mV^\top$ be the SVD of $\mW$, and let
$\widehat{\mW}=\mU_{1:k}\mSigma_{1:k}\mV_{1:k}^\top$ be its rank-$k$ truncated SVD approximation.
Assume DyRA approximates $\mX\mW_{\mathrm{res}}$ with rank $r$ by one ALS iteration, initialized with
$\mR^{(0)}=\mV_{k+1:k+r}^\top$.
Then the resulting approximation can be written as
\begin{equation}
\begin{aligned}
    \mX\mW
    &\approx
    \mX
    \begin{bmatrix}
        \mU_{1:k}\mSigma_{1:k} &
        \mU_{k+1:k+r}\mSigma_{k+1:k+r}
    \end{bmatrix}
    \begin{bmatrix}
        \mV_{1:k}^\top \\
        \mR^{(1)}
    \end{bmatrix},
\end{aligned}
\end{equation}
where $\mR^{(1)}
=
(\mA_r^{\mathrm{res}\top}\mX^\top\mX\mA_r^{\mathrm{res}})^{-1}
\mA_r^{\mathrm{res}\top}\mX^\top\mX\mW_{\mathrm{res}}$
and $\mA_r^{\mathrm{res}}=\mU_{k+1:k+r}\mSigma_{k+1:k+r}$.
\end{proposition}

The proposition shows that DyRA uses the same additional left directions as the static rank-$(k+r)$ SVD approximation,
$\mU_{k+1:k+r}\mSigma_{k+1:k+r}$.
The difference lies in how these directions are paired on the right side. A static rank-$(k+r)$ approximation uses the fixed weight-only basis
$\mV_{k+1:k+r}^\top$.
In contrast, DyRA replaces this basis with $\mR^{(1)}$, which is optimized for the current input through $\mX^\top\mX$.
Thus, DyRA is not merely increasing the static rank; it constructs an input-adaptive residual subspace tailored to the realized product $\mX\mW_{\mathrm{res}}$.

This distinction becomes clearer in the isotropic case $\mX^\top\mX=\mI$.
In this case, $\mR^{(1)}=\mV_{k+1:k+r}^\top$, so DyRA reduces to the static rank-$(k+r)$ approximation.
When $\mX^\top\mX\neq\mI$, however, the current input reweights feature directions, and $\mR^{(1)}$ generally differs from the weight-only basis.
This is precisely where DyRA gains input adaptivity. The proof is provided in Appendix~\ref{sup:proof_subspace}.

\paragraph{Computational Cost }
We now quantify the computational overhead of constructing the input-adaptive subspace.
Let $\mX\in\mathbb{R}^{n\times c_1}$, $\mW_{\mathrm{res}}\in\mathbb{R}^{c_1\times c_2}$, and let $r$ denote the rank of the dynamic residual correction. 
As shown in Eq.~\ref{eq:init_L_update}, the left-factor update multiplies $\mX$ by a slim $c_1\times r$ matrix, costing $O(nc_1r)$. 
The right-factor update projects the residual product onto the $r$-dimensional subspace defined by $\mL^{(1)}$, with dominant cost $O(nc_1r + c_1c_2r)$. Including the final multiplication $\mL^{(1)}\mR^{(1)}$, the dynamic residual correction costs $O(nc_1r + c_1c_2r + nc_2r)$. Combined with the $O(nk(c_1+c_2))$ cost of the rank-$k$ static approximation, DyRA totals $O\big(nk(c_1+c_2) + r\big(n(c_1+c_2) + c_1c_2\big)\big)$.
Since $r$ is typically much smaller than $c_1$ and $c_2$, this remains substantially cheaper than the dense product $\mX\mW$, which costs $O(nc_1c_2)$. Compared with simply increasing the static rank to $k+r$, DyRA incurs an extra $O(c_1c_2r)$ term, which is the cost of constructing the input-adaptive subspace.

\begin{figure}[!t]
    \begin{center}
    \includegraphics[width=1.0\linewidth]{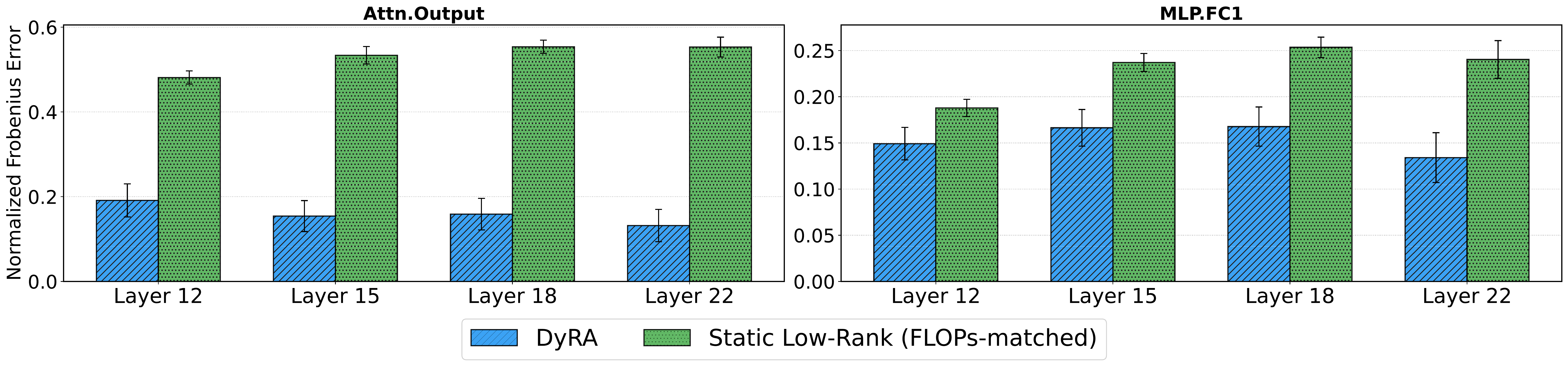}
    \end{center}
    \vspace{-0.5em}
    \caption{Remaining normalized Frobenius error of DyRA and a static low-rank baseline under a matched FLOPs budget. The error is measured between the original dense output and the approximated output across ViT layers.}
    \label{fig:flops_matched_error}
    \vspace{-1.0em}
\end{figure}
\paragraph{FLOPs-Matched Comparison}\label{paragraph:error_comparison}
Given this overhead, we examine whether DyRA's gains persist under a matched computational budget.
To this end, we compare DyRA against a FLOPs-matched static low-rank weight approximation baseline on ViT.
Starting from a rank-$k$ approximation of $\mW$, DyRA applies a rank-$r$ dynamic correction to reduce the remaining output error.
For the static baseline, we instead increase the rank from $k$ to $k+r+r'$, where the additional rank $r'$ accounts for the $O(c_1c_2r)$ overhead of the dynamic correction, so that both methods use the same FLOPs budget \footnote{The matching is done at the level of exact FLOPs counts rather than asymptotic order, so $r'$ also absorbs the constant and lower-order terms omitted in the $O(\cdot)$ notation.}.
Figure~\ref{fig:flops_matched_error} shows the remaining normalized Frobenius error across different ViT layers.
Under this FLOPs-matched setting, DyRA consistently achieves lower error than the static low-rank baseline.
These results show that, under the same computational budget, input-adaptive residual subspaces approximate the output more faithfully than higher-rank static approximations. Additional results for the remaining attention and MLP layers are provided in Appendix~\ref{sup:flops_matched_error}.

\subsection{Implementation of DyRA}
\label{sec:dyra_kernel}


To minimize GPU memory traffic, we implement DyRA with custom Triton~\cite{triton} kernels.
Rather than materializing the reconstructed residual $\mL\mR$ in HBM, we fuse the
low-rank reconstruction with the output merge: given
$\mL \in \mathbb{R}^{B \times T \times r}$ and
$\mR \in \mathbb{R}^{B \times r \times N}$, the kernel accumulates
$\alpha \mL\mR$ on-chip and directly writes
$\mY_{\mathrm{base}}+\alpha\mL\mR$ to the output tensor.
For linear layers, we further fuse the base projection, bias addition, and
residual correction, computing
$\mY=\mX_{\mathrm{tmp}}\mA^\top+\vb+\alpha\mL\mR$
within a single Triton kernel.
This eliminates intermediate tensors and reduces HBM accesses.

The ALS update requires solving an $r \times r$ symmetric positive definite system for each batch instance. For small $r$, we use a Triton Cholesky solver, keeping the factorization and triangular solves in on-chip cache/registers. For a larger $r$, Cholesky becomes too sequential to fully utilize GPU parallelism, so we use a Newton--Schulz iterative solver. This replaces dependency-heavy factorization with matrix-multiplication-based updates, improving the GPU's parallelism and tensor-core utilization.

\section{Experimental Results}
\begin{figure}[t]
    \begin{center}
    \includegraphics[width=1.0\linewidth]{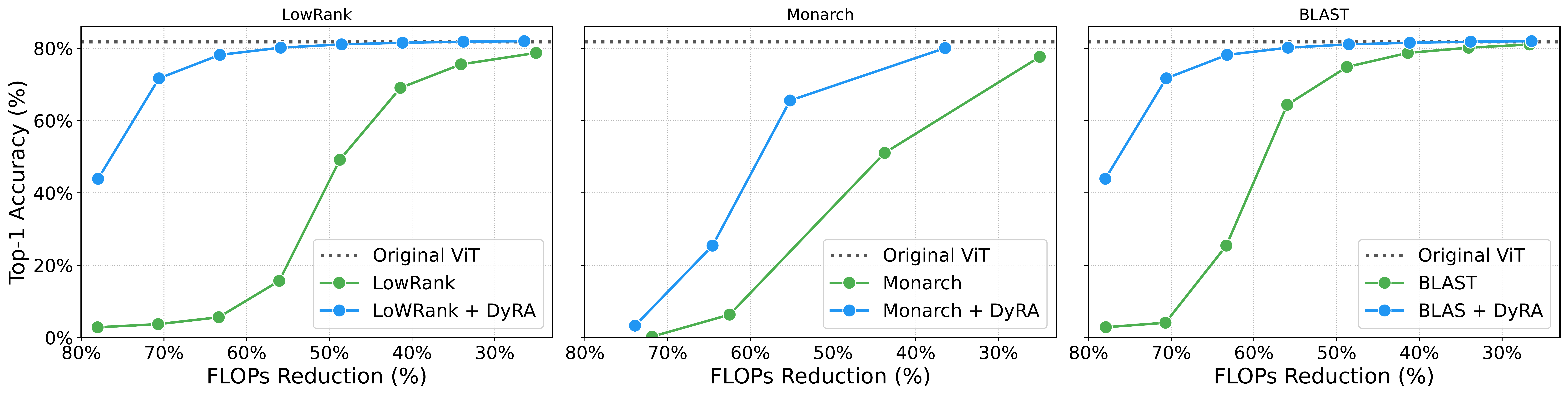}
    \end{center}
    \caption{ImageNet classification accuracy at different FLOPs reduction levels for ViT-L. DyRA consistently improves the accuracy–efficiency trade-off across low-rank, Monarch, and BLAST.}
    \vspace{-1.0em}
    \label{fig:vit_imagenet}
\end{figure}

We validate DyRA by integrating it with several structured weight approximations, including low-rank, Monarch~\cite{monarch}, and BLAST~\cite{blast}. 
We first evaluate DyRA as a test-time model approximation/acceleration technique applied to pre-trained models (without any fine-tuning) across vision, language, and speech tasks, demonstrating consistent improvements in downstream performance under matched computational budgets (Section~\ref{section:zeroshot}). 
We then measure end-to-end latency to assess whether these computational savings translate into practical speedups on a GPU (Section~\ref{section:efficiency}). 
Finally, we study settings where additional adaptation is available, showing that DyRA can be combined with post-training calibration or fine-tuning to further improve performance (Section~\ref{section:adaptation}). 
All experimental details are provided in Appendix~\ref{sup:exp_detail}.

\subsection{Test-Time Model Approximation and Acceleration Across Domains}
\label{section:zeroshot}

\paragraph{Image Classification with ViT}
We evaluate DyRA on ImageNet~\cite{imagenet} classification using a pre-trained ViT-L model~\cite{vit}. We apply structured weight approximations to all dense linear layers in both the attention and MLP blocks, and measure top-1 accuracy under different FLOPs budgets. For DyRA, we fix the dynamic residual rank to $r=8$ and vary the budget of the static weight approximation to control the overall FLOPs. 
As shown in Figure~\ref{fig:vit_imagenet}, DyRA consistently improves the accuracy--efficiency trade-off across all structured baselines.
In particular, at around 55\% FLOPs reduction, the low-rank and BLAST baselines incur accuracy drops of 66.1\% and 17.4\%, respectively, whereas combining them with DyRA reduces the drops to 2.4\% and 1.6\%.

\paragraph{Dense Prediction Tasks with DINO}
Beyond image-level classification, we further show that DyRA preserves fine-grained representations required for dense prediction.
To this end, we evaluate semantic segmentation on ADE20K~\cite{ade20k} and depth estimation on NYUv2~\cite{nyuv2} using DINOv3~\cite{dinov3}.
We apply DyRA to all linear layers and compare it against the corresponding static weight approximation baselines.
As shown in Figure~\ref{fig:dino_dense}, DyRA achieves a substantially better accuracy--efficiency trade-off across both tasks.
In particular, for semantic segmentation at 60\% FLOPs reduction, DyRA reduces mIoU by only 1.7 points relative to the original model, preserving most of its dense prediction performance.
In contrast, the static weight approximation baseline suffers a 54.1 point drop in mIoU, indicating a severe degradation in fine-grained representations.

\begin{figure}[t]
    \centering
    \begin{subfigure}{1.0\linewidth}
        \centering
        \includegraphics[width=1.0\linewidth]{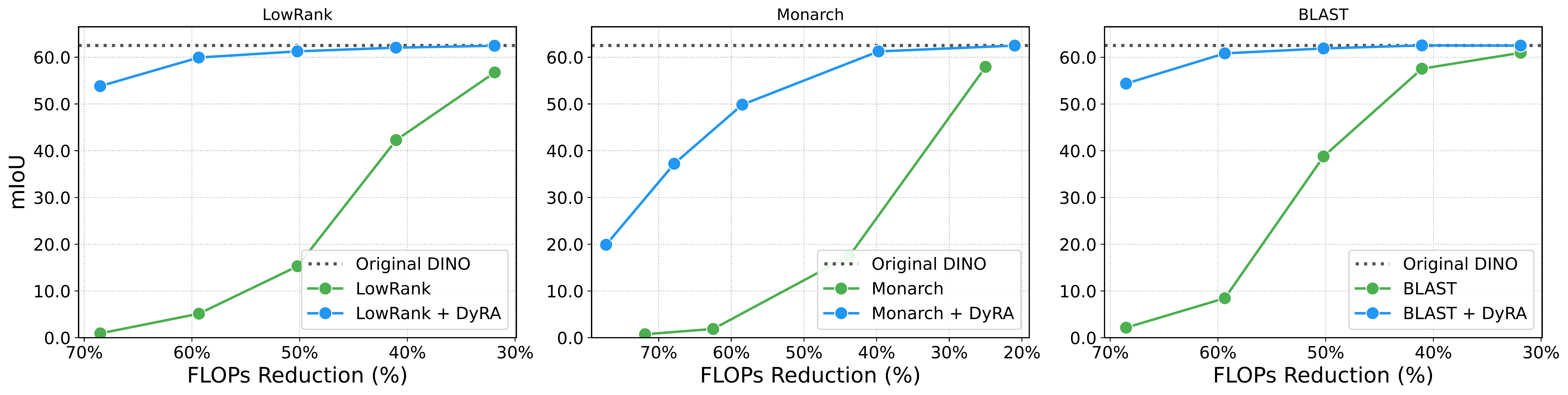}
        \caption{Semantic Segmentation on ADE20K.}
        \label{fig:vit_imagenet_main}
    \end{subfigure}


    \begin{subfigure}{1.0\linewidth}
        \centering
        \includegraphics[width=1.0\linewidth]{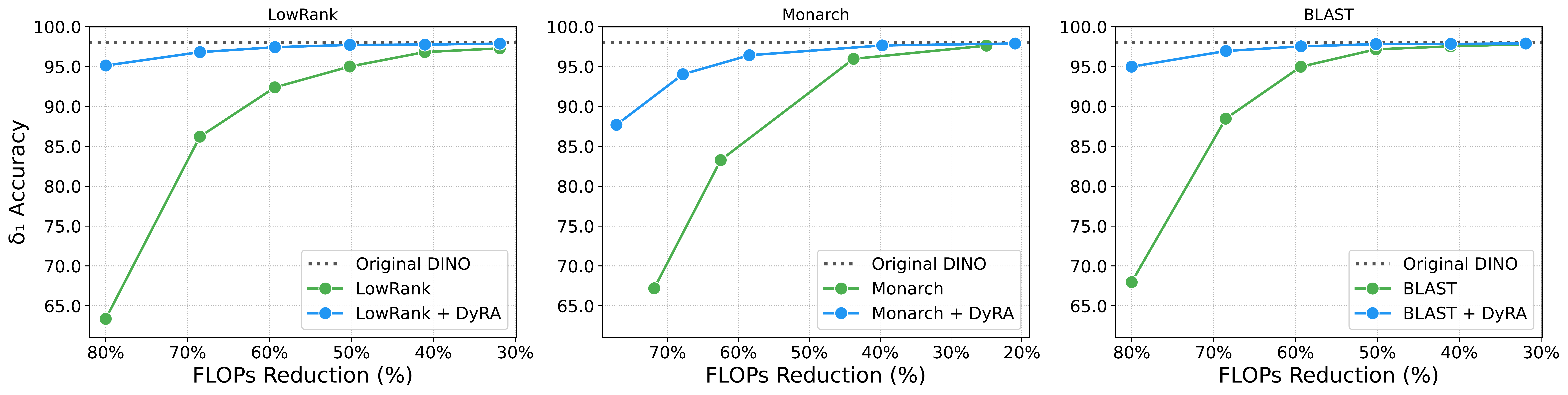}
        \caption{Depth Estimation on NYUv2.}
        \label{fig:second_result}
    \end{subfigure}
    \caption{DyRA vs.\ static weight approximations on dense prediction with DINOv3. DyRA achieves a superior accuracy--efficiency trade-off for semantic segmentation and depth estimation.}
    \vspace{-1.5em}
    \label{fig:dino_dense}
\end{figure}

\paragraph{Speech Recognition}

\begin{wraptable}{r}{0.34\textwidth}
\vspace{-1.4em}
\centering
\caption{WER under different FLOPs reduction ratios (FRR).}
\label{tab:asr}
\resizebox{\linewidth}{!}{
\begin{tabular}{c|l|c}
\hline
FRR & Method & WER($\downarrow$) \\
\hline
0\%
& Original & 1.92 \\
\hline
\multirow{4}{*}{25\%}
& Low-rank        & 2.21 \\
& Low-rank + DyRA & \textbf{2.08} \\
& BLAST           & 2.07 \\
& BLAST + DyRA    & \textbf{2.03} \\
\hline
\multirow{4}{*}{50\%}
& Low-rank        & 24.54 \\
& Low-rank + DyRA & \textbf{4.89} \\
& BLAST           & 6.48 \\
& BLAST + DyRA    & \textbf{4.27} \\
\hline
\end{tabular}
}
\vspace{-1em}
\end{wraptable}

We evaluate DyRA on automatic speech recognition using a pre-trained Wav2Vec 2.0 Large model~\cite{wav2vec2}.
We apply DyRA to all MLP layers and the query/key projections, and measure word error rate (WER) under different FLOPs reduction ratios.
As shown in Table~\ref{tab:asr} on the LibriSpeech Clean test set~\cite{libri}, DyRA consistently achieves lower WER than static weight approximation baselines under matched computational budgets.
These results demonstrate that, similar to the vision domains, combining DyRA with structured weight approximations is more effective than using weight approximation alone, and generalizes well to the audio domain.

\paragraph{Language Modeling}
We evaluate DyRA on language modeling using LLaDA-8B~\cite{llada}, a large language diffusion model.
We apply DyRA to all MLP layers and the query/key projections, and evaluate performance on HumanEval~\cite{humaneval} for code generation and GSM8K~\cite{gsm8k} for mathematical reasoning.
As shown in Figure~\ref{fig:large}, the low-rank baseline performs poorly on both tasks, with accuracy below 4\%.
In contrast, DyRA improves performance by more than 20 percentage points in both cases under the same computational budget, achieving results substantially closer to the original model.
These results demonstrate that DyRA makes structured weight approximation usable for large language diffusion models, even when the FLOPs-matched low-rank baseline largely fails.

\subsection{Runtime Analysis}
\label{section:efficiency}
\begin{wrapfigure}{r}{0.45\textwidth}
    \vspace{-1.8em}
    \centering
    \includegraphics[width=0.43\textwidth]{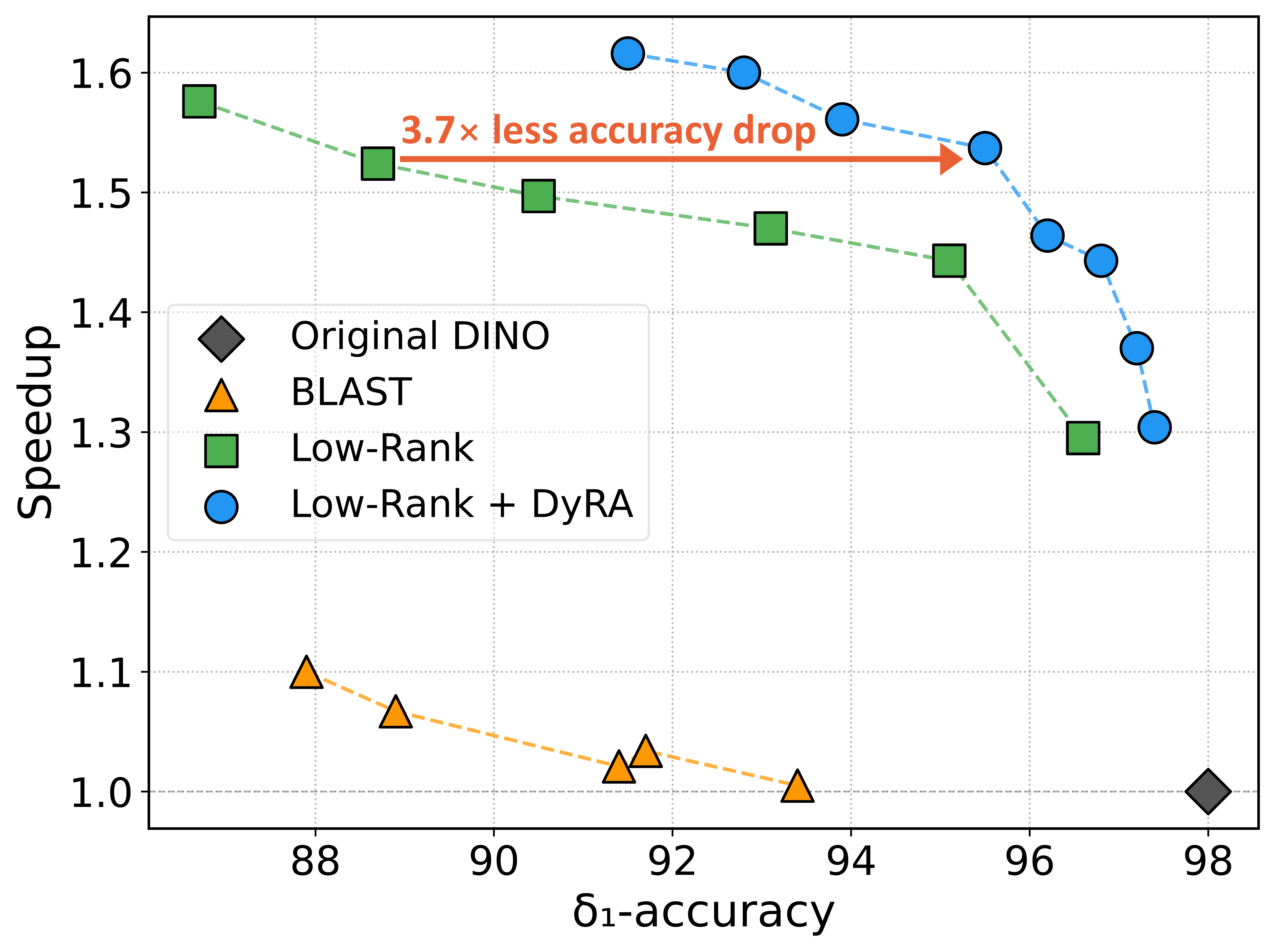}
    \caption{End-to-end latency vs.\ accuracy on DINOv3 depth estimation.}
    \label{fig:runtime}
    \vspace{-1.4em}
\end{wrapfigure}
We evaluate whether the computational savings of DyRA translate into actual end-to-end speedups on a GPU.
To this end, we apply DyRA to DINOv3 and compare it against static weight approximation baselines, including low-rank and BLAST.
Following Section~\ref{sec:dyra_kernel}, we use the fused Triton kernels for DyRA.
For the weight-only baselines, we use the faster option between the optimized structured kernels in~\cite{lowrank_kernel} and \texttt{torch.compile()} optimization of PyTorch~\cite{pytorch}.
We measure the end-to-end latency of depth estimation on the NVIDIA A40 GPU.

Figure~\ref{fig:runtime} shows the relative speedup across different depth-estimation accuracy levels. DyRA achieves a $1.3\times$ end-to-end speedup with only a small accuracy degradation of about $0.6\%$, demonstrating that its FLOPs reduction translates into practical runtime improvement. Compared with the static low-rank baseline, DyRA provides a substantially better accuracy--latency trade-off. It achieves larger speedups at similar accuracy levels, while the low-rank baseline incurs a $3.7\times$ larger accuracy drop to reach a $1.52\times$ speedup. Overall, these results show that DyRA provides practical speedups on the NVIDIA A40 GPU while maintaining a more favorable balance between speed and accuracy than static weight approximations.

\subsection{Further Evaluations with Augmented Methods}
\label{section:adaptation}
\paragraph{Comparison with Input-aware Decomposition}
We compare DyRA with input-aware decomposition~\cite{asvd,svdllm}, a post-training method that leverages calibration data to compute low-rank weight factors better aligned with output activations.
Figure~\ref{fig:combined} reports results on LLaDA for HumanEval and GSM8K, and on ViT for ImageNet classification.
Compared with vanilla low-rank decomposition, input-aware decomposition recovers part of the performance loss across these settings, confirming that activation information is useful for structured approximation.

However, since the decomposed weights are fixed after calibration, input-aware decomposition cannot adapt to the activations of each test sample.
DyRA provides this missing input-dependent correction at inference time, achieving a better accuracy--efficiency trade-off than input-aware decomposition alone.
Moreover, because DyRA operates as a correction on top of the structured computation, it can be combined with input-aware decomposition.
In LLaDA, combining DyRA with input-aware decomposition further improves performance on both HumanEval and GSM8K.
Similarly, on ImageNet classification at around 70\% FLOPs reduction, combining DyRA with input-aware decomposition further improves accuracy from 65.4\% to 68.8\% over DyRA with vanilla low-rank decomposition, showing that DyRA can further benefit from stronger calibrated static approximations.

\begin{figure}[t]
    \centering
    \begin{subfigure}[t]{0.67\textwidth}
        \centering
        \includegraphics[height=0.18\textheight]{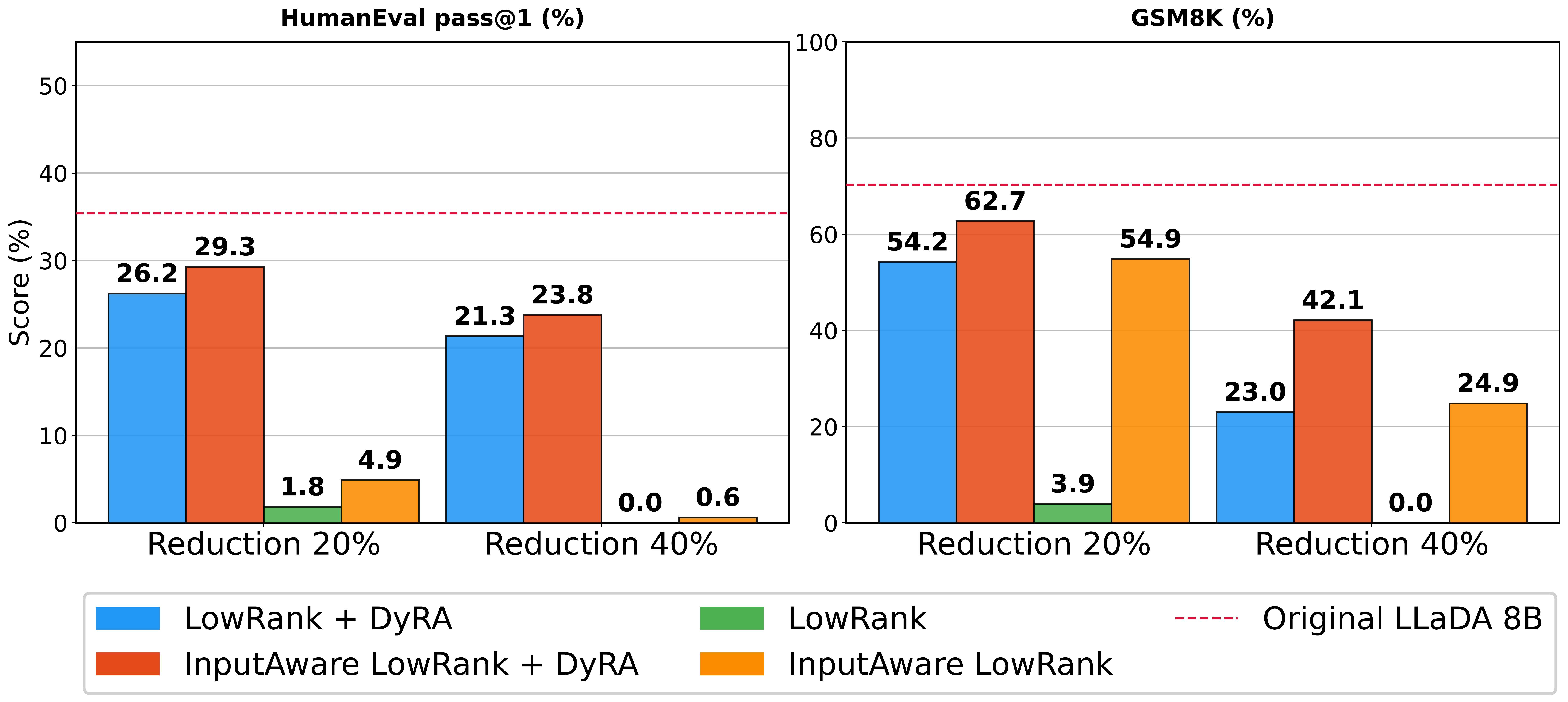}
        \caption{LLaDA 8B on HumanEval and GSM8K}
        \label{fig:large}
    \end{subfigure}
    \hfill
    \begin{subfigure}[t]{0.32\textwidth}
        \centering
        \includegraphics[height=0.18\textheight]{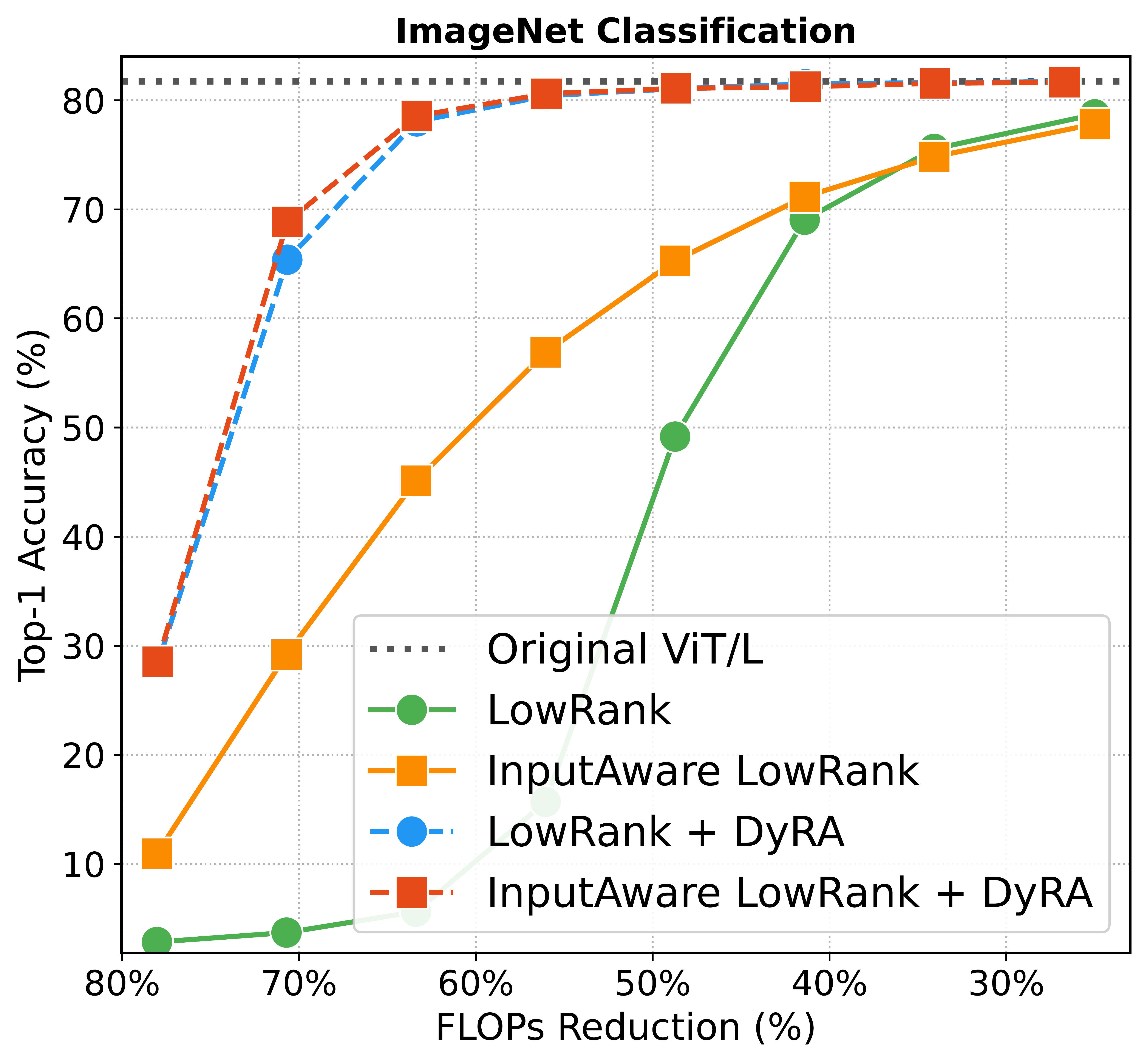}
        \caption{ViT-L on ImageNet-1K}
        \label{fig:small}
    \end{subfigure}
    \caption{Comparison between DyRA and input-aware decomposition on LLaDA-8B and ViT-L. DyRA generally achieves better performance, and can further improve results by using input-aware decomposition as the static approximation.}
    \label{fig:combined}
    \vspace{-1em}
\end{figure}

\paragraph{Fine-tuning Results}
\begin{wrapfigure}{r}{0.35\textwidth}
    \vspace{-1em}
    \centering
    \includegraphics[width=0.34\textwidth]{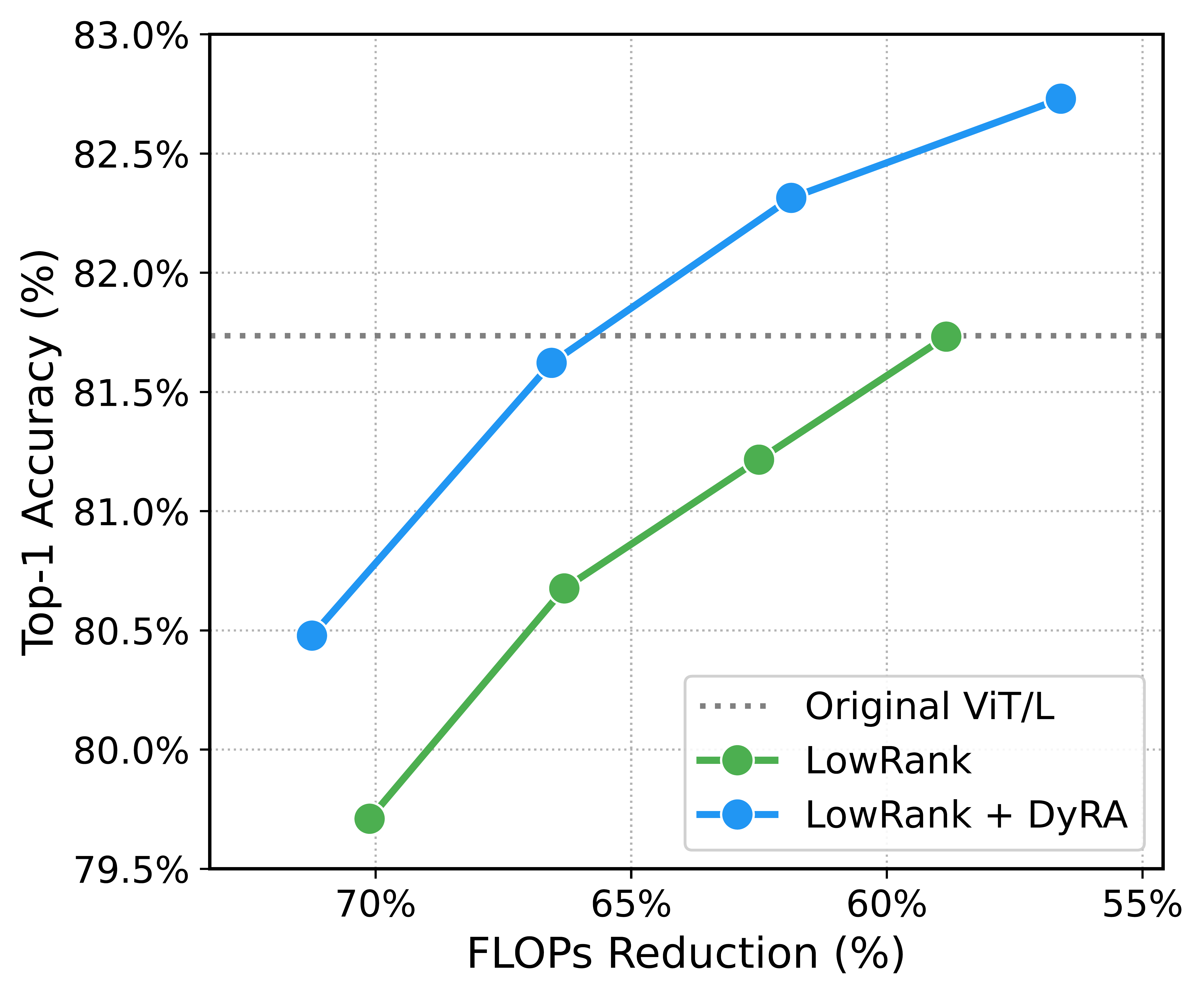}
    \caption{Fine-tuning results on ImageNet classification with ViT-L.}
    \label{fig:finetune}
    \vspace{-1em}
\end{wrapfigure}
Beyond combining DyRA with input-aware decomposition, we further evaluate DyRA under direct fine-tuning on ImageNet classification. In this experiment, we replace all linear layers of a ViT-L model with DyRA-augmented low-rank approximations and fine-tune the resulting model on ImageNet for 15 epochs. For comparison, we fine-tune a low-rank weight-only baseline under the same setting. As shown in Figure~\ref{fig:finetune},  consistent with the test-time approximation evaluation in prior subsections, DyRA with fine-tuning outperforms the static weight approximation baseline across different FLOPs budgets. Notably, at around 60\% FLOPs reduction, DyRA even slightly surpasses the original dense model while using substantially fewer FLOPs. These results suggest that DyRA remains effective after fine-tuning.

\paragraph{Ablation Study}
We analyze key components of DyRA: the residual rank $r$ and the number of ALS iterations $t$ for an ablation study. We conduct experiments on ViT image classification by combining DyRA with a low-rank weight approximation and varying each hyperparameter. Figure~\ref{fig:ablation} (left) reports the effect of the residual rank, evaluated with $r \in \{4,8,16,24\}$. DyRA consistently outperforms the corresponding low-rank static weight approximation baseline across all choices of $r$, showing that allocating even a small budget to the dynamic residual is beneficial and that performance is relatively robust to the choice of $r$. We then evaluate the effect of the ALS iteration count with $t \in \{1,1.5,2\}$ in Figure~\ref{fig:ablation} (right), where $t=1.5$ denotes the update sequence $L,R,L$. It shows that all configurations substantially improve over the static approximation baseline. Notably, $t=1$ already achieves strong performance, suggesting that most of the gain comes from a single iteration. This supports our default choice of $t=1$, which provides a favorable trade-off between accuracy and computational cost.

\begin{figure}[t]
    \begin{center}
    \includegraphics[width=1.0\linewidth]{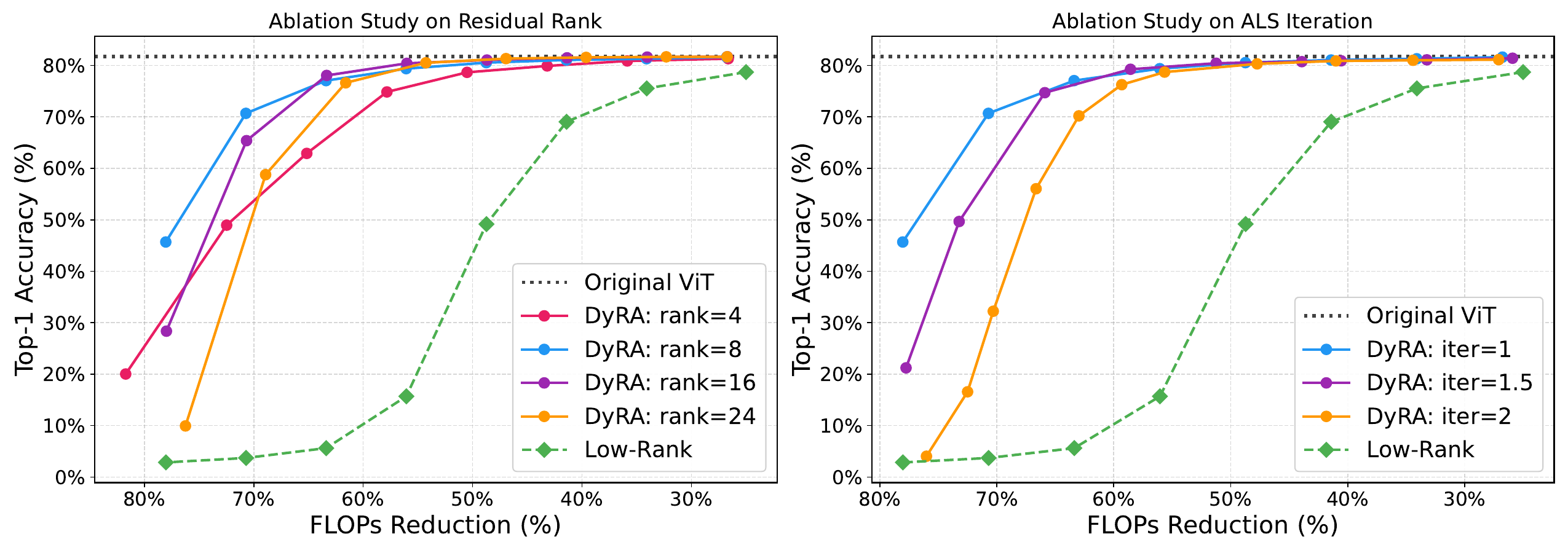}
    \end{center}
    \caption{Ablation study on ImageNet classification with ViT-L. Left: DyRA with different residual ranks $r$. Right: DyRA with different ALS iteration counts $t$.}
    \vspace{-1em}
    \label{fig:ablation}
\end{figure}

\section{Related Works}
\paragraph{Structured Matrices for Efficient DNN Inference}
Structured matrices have been widely used to reduce the cost of dense linear layers in deep neural networks. Prior studies~\cite{welore, sgd_lowrank, lowrank_overview,lora} have observed that pretrained weight matrices often exhibit low-rank structure, motivating low-rank approximation as a simple compression strategy~\cite{transformer_lowrank}. More expressive block-structured variants, such as Monarch~\cite{monarch} and BLAST~\cite{blast}, improve the trade-off between efficiency and accuracy, while methods such as GBLR~\cite{gblr} and Einsum Linear~\cite{einsum_linear} learn efficient weight structures directly. Another line of work incorporates activation information from calibration data to obtain input-aware weight decompositions~\cite{asvd,svdllm,feature_lowrank, drone}. However, these methods still produce a static approximation before inference, and cannot adapt to the sample-dependent error induced by the current input. Our method addresses this limitation by approximating the realized output error on the fly for each input. Importantly, our method can be combined with advanced structured or input-aware decompositions to obtain a more faithful matrix multiplication under a fixed computational budget.

\paragraph{Low-Rank Approximation via Alternating Least Squares}
Alternating least squares (ALS) has been studied as an efficient method for low-rank approximation and reconstruction, as it reduces a rank-constrained problem to a sequence of tractable least-squares subproblems. Prior work has used ALS for low-rank matrix reconstruction from underdetermined measurements~\cite{als_lowrank}, structured low-rank recovery with ADMM-style updates~\cite{als_lowrank_good}, and spectral analyses connected to orthogonal iteration~\cite{als_bound}. In contrast, we use ALS to approximate matrix products directly, targeting efficient matrix multiplication rather than reconstruction of a single matrix.

\section{Conclusion}
This work shows that the accuracy--efficiency trade-off of matrix multiplication in pre-trained DNNs can be improved by approximating the realized output product, rather than only the weight matrix. We formulate this output-domain approximation as a low-rank factorization problem and solve it efficiently with ALS during model inference. Building on this formulation, we propose DyRA, which augments static weight approximations with an input-adaptive residual correction. Our experiments show that this correction consistently improves the accuracy--efficiency trade-off across low-rank, Monarch, and BLAST approximations. In particular, DyRA preserves dense prediction performance under aggressive FLOPs reduction and translates these savings into practical speedups on a GPU.

\paragraph{Limitations and Future Work.} DyRA is designed to accelerate dense matrix multiplication by combining dynamic residual approximation with efficient structured weight approximations. As a result, it is most effective in regimes with large matrix--matrix multiplications, whereas token-by-token vector--matrix workloads such as autoregressive decoding are not its primary target application. Nevertheless, DyRA remains applicable to the prefill stage of autoregressive language models, where the full prompt induces large matrix multiplications. We provide supporting results in the supplementary material (Appendix~\ref{sup:llama}).

Finally, as a test-time acceleration approach, DyRA offers additional flexibility compared to static approximation methods whose rank must be fixed offline. In principle, the dynamic residual rank can be adjusted based on the input or computational budget at inference time. Exploring such adaptive rank allocation is left for future work.


\section*{Acknowledgments}
This work was in part supported by MSTAR automotive microelectronics research award \#N036453, and also by the Center on Cognitive Multispectral Sensors (COGNISENSE), one of seven centers in Joint University Microelectronics Program (JUMP) 2.0, a Semiconductor Research Corporation (SRC) program sponsored by Defense Advanced Research Projects Agency (DARPA). Daewon Chae would like to acknowledge the support from the Hyundai Motor Chung Mong-Koo Foundation.

\bibliographystyle{plainnat}
\bibliography{ref}


\newpage
\appendix

\section{DyRA Algorithm}
Algorithm~\ref{alg:dyra} summarizes the overall DyRA procedure, including offline preparation and inference-time dynamic residual correction.

\begin{algorithm}[h]
\caption{DyRA with One-Step ALS Residual Approximation}
\label{alg:dyra}
\begin{algorithmic}[1]

\Procedure{Offline Preparation}{$\mW, k, r, \lambda$}
    \State Compute a structured approximation $\widehat{\mW}$ of $\mW$
    \State $\mW_{\mathrm{res}} \gets \mW - \widehat{\mW}$
    \State Initialize and fix $\mR^{(0)} \in \mathbb{R}^{r \times c_2}$
    \State Precompute $\mP_L \gets
    \left(\mW_{\mathrm{res}}\mR^{(0)\top}\right)
    \left(\mR^{(0)}\mR^{(0)\top}+\lambda\mI_r\right)^{-1}$
    \State \Return $\widehat{\mW}, \mW_{\mathrm{res}}, \mP_L$
\EndProcedure

\Statex

\Procedure{Inference}{$\mX, \widehat{\mW}, \mW_{\mathrm{res}}, \mP_L, \lambda$}
    \State $\mY_{\mathrm{base}} \gets \mX\widehat{\mW}$
    \State $\mL^{(1)} \gets \mX\mP_L$
    \State $\mR^{(1)} \gets
    \left(\mL^{(1)\top}\mL^{(1)}+\lambda\mI_r\right)^{-1}
    \left(\mL^{(1)\top}\mX\right)\mW_{\mathrm{res}}$
    \State $\widehat{\mY} \gets \mY_{\mathrm{base}}+\mL^{(1)}\mR^{(1)}$
    \State \Return $\widehat{\mY}$
\EndProcedure

\end{algorithmic}
\end{algorithm}

\section{Background on Structured Matrices}
\paragraph{Low-Rank Approximation}
Low-rank approximation represents a weight matrix $\mW\in\mathbb{R}^{d\times m}$ as the product of two smaller matrices,
\[
    \widehat{\mW}=\mL\mR,
    \qquad
    \mL\in\mathbb{R}^{d\times r},\quad
    \mR\in\mathbb{R}^{r\times m},
\]
where $r\ll\min(d,m)$.
A standard choice is the rank-$r$ truncated singular value decomposition (SVD).
Given $\mW=\mU\mSigma\mV^\top$, the approximation is
\[
    \widehat{\mW}
    =
    \mU_r\mSigma_r\mV_r^\top,
    \qquad
    \mL=\mU_r\mSigma_r,\quad
    \mR=\mV_r^\top .
\]
This approximation keeps the top-$r$ singular components of $\mW$, providing the best rank-$r$ approximation under the Frobenius norm.

\paragraph{Monarch Matrix}
Monarch matrices~\cite{monarch} represent a weight matrix using two block-diagonal factors with an intermediate permutation.
For $\mW\in\mathbb{R}^{N\times N}$ with $N=b_1b_2$, the approximation is written as
\[
    \widehat{\mW}
    =
    \mP \mL \mP^\top \mR,
    \qquad
    \mL=\mathrm{diag}(\mL_1,\dots,\mL_{b_2}),
    \qquad
    \mR=\mathrm{diag}(\mR_1,\dots,\mR_{b_1}),
\]
where $\mL_i\in\mathbb{R}^{b_1\times b_1}$ for $i\in[b_2]$ and
$\mR_j\in\mathbb{R}^{b_2\times b_2}$ for $j\in[b_1]$.
Here, $\mP\in\mathbb{R}^{N\times N}$ is a fixed permutation matrix that reorders the block structure, so that the overall product induces a structured block-wise pattern in $\widehat{\mW}$.
In practice, the factors can be obtained by reshaping the dense weight into blocks, applying a truncated SVD to each block, and rearranging the resulting factors.

\paragraph{BLAST}
BLAST~\cite{blast} represents a weight matrix using shared block-row and block-column bases.
For $\mW\in\mathbb{R}^{d\times m}$, we partition the matrix into a grid of blocks $\mW_{i,j}\in\mathbb{R}^{p\times q}$.
Each block is parameterized as
\[
    \widehat{\mW}_{i,j}
    =
    \mU_i \mS_{i,j} \mV_j^\top,
    \qquad
    \mU_i\in\mathbb{R}^{p\times r},\quad
    \mS_{i,j}\in\mathbb{R}^{r\times r},\quad
    \mV_j\in\mathbb{R}^{q\times r}.
\]
Here, $\mU_i$ is shared by all blocks in the same block row, and $\mV_j$ is shared by all blocks in the same block column.
The middle factor $\mS_{i,j}$ is specific to each block and controls how the shared row and column bases are combined.
This parameterization captures low-dimensional structure at the block level while reusing bases across the matrix.
In practice, the BLAST factors are obtained by minimizing the reconstruction error between $\mW$ and $\widehat{\mW}$.

\paragraph{Input-Aware Decomposition}
Input-aware decomposition aims to better preserve output activations by incorporating calibration inputs into the weight approximation.
Given a calibration set $\mathcal{S}=\{\mX_s\}_{s=1}^{|\mathcal{S}|}$, these methods approximate $\mW$ by reducing the output-space discrepancy measured on the calibration set,
$\sum_{\mX_s\in\mathcal{S}}\|\mX_s\mW-\mX_s\widehat{\mW}\|_F^2$.
A common formulation introduces an invertible transformation matrix $\mS$ and rewrites the weight as
$\mW=(\mW\mS)\mS^{-1}$.
Truncated SVD is then applied to the transformed weight,
\[
    \mW\mS
    \approx
    \mU_r\mSigma_r\mV_r^\top,
\]
yielding the approximation
\begin{equation}
    \widehat{\mW}
    =
    \mU_r\mSigma_r\mV_r^\top\mS^{-1}.
    \label{eq:aware_svd}
\end{equation}

The transform $\mS$ determines how calibration activations influence the decomposition.
For example, $\mS$ can be chosen as a diagonal scaling matrix based on channel-wise activation magnitudes~\cite{asvd}, or as a Cholesky factor of the input covariance matrix computed from the calibration set~\cite{asvd,svdllm}.

\section{Proof of Proposition~\ref{proposition1}}
\label{sup:proof_subspace}
\begin{proof}
We start from the one-step DyRA approximation in Eq.~\ref{eq:drama_final},
\begin{equation}
    \mX\mW
    \approx
    \mX\widehat{\mW}+\mL^{(1)}\mR^{(1)} .
\end{equation}
In the low-rank case, the static approximation is
\begin{equation}
    \widehat{\mW}
    =
    \mU_{1:k}\mSigma_{1:k}\mV_{1:k}^\top,
\end{equation}
and the residual is
\begin{equation}
    \mW_{\mathrm{res}}
    =
    \mW-\widehat{\mW}
    =
    \mU_{k+1:}\mSigma_{k+1:}\mV_{k+1:}^\top .
\end{equation}

We initialize the right factor as $\mR^{(0)}=\mV_{k+1:k+r}^\top$. Since the singular vectors are orthonormal, $\mR^{(0)}\mR^{(0)\top}=\mI$. Substituting this into the left-factor update in Eq.~\ref{eq:L_update} yields
\begin{equation}
\begin{aligned}
    \mL^{(1)}
    &=
    \mX
    \left(\mW_{\mathrm{res}}\mR^{(0)\top}\right)
    \left(\mR^{(0)}\mR^{(0)\top}\right)^{-1} \\
    &=
    \mX
    \left(
    \mU_{k+1:}\mSigma_{k+1:}\mV_{k+1:}^\top
    \mV_{k+1:k+r}
    \right) \\
    &=
    \mX\mU_{k+1:k+r}\mSigma_{k+1:k+r}.
\end{aligned}
\end{equation}
Let $\mA_r^{\mathrm{res}}=\mU_{k+1:k+r}\mSigma_{k+1:k+r}$, so that $\mL^{(1)}=\mX\mA_r^{\mathrm{res}}$.

Substituting this into the right-factor update in Eq.~\ref{eq:R_update}, we obtain
\begin{equation}
\begin{aligned}
    \mR^{(1)}
    &=
    \left(\mL^{(1)\top}\mL^{(1)}\right)^{-1}
    \left(\mL^{(1)\top}\mX\right)\mW_{\mathrm{res}} \\
    &=
    \left(
    \mA_r^{\mathrm{res}\top}\mX^\top\mX\mA_r^{\mathrm{res}}
    \right)^{-1}
    \mA_r^{\mathrm{res}\top}\mX^\top\mX\mW_{\mathrm{res}} .
\end{aligned}
\end{equation}

Plugging the resulting one-step factors into Eq.~\ref{eq:drama_final}, we obtain
\begin{equation}
\begin{aligned}
    \mX\mW
    &\approx
    \mX\widehat{\mW}
    +
    \mL^{(1)}\mR^{(1)} \\
    &=
    \mX\mU_{1:k}\mSigma_{1:k}\mV_{1:k}^\top
    +
    \mX\mU_{k+1:k+r}\mSigma_{k+1:k+r}\mR^{(1)} \\
    &=
    \mX
    \left(
    \mU_{1:k}\mSigma_{1:k}\mV_{1:k}^\top
    +
    \mU_{k+1:k+r}\mSigma_{k+1:k+r}\mR^{(1)}
    \right) \\
    &=
    \mX
    \begin{bmatrix}
        \mU_{1:k}\mSigma_{1:k} &
        \mU_{k+1:k+r}\mSigma_{k+1:k+r}
    \end{bmatrix}
    \begin{bmatrix}
        \mV_{1:k}^\top \\
        \mR^{(1)}
    \end{bmatrix}.
\end{aligned}
\end{equation}
This proves the stated form.

Finally, if $\mX^\top\mX=\mI$, then
\begin{equation}
\begin{aligned}
    \mR^{(1)}
    &=
    \left(
    \mA_r^{\mathrm{res}\top}\mA_r^{\mathrm{res}}
    \right)^{-1}
    \mA_r^{\mathrm{res}\top}\mW_{\mathrm{res}}.
\end{aligned}
\end{equation}
Using orthogonality of singular vectors,
\[
\mA_r^{\mathrm{res}\top}\mA_r^{\mathrm{res}}
=
\mSigma_{k+1:k+r}^{2}, \qquad
\mA_r^{\mathrm{res}\top}\mW_{\mathrm{res}}
=
\mSigma_{k+1:k+r}^{2}\mV_{k+1:k+r}^\top,
\]
which gives $\mR^{(1)}=\mV_{k+1:k+r}^\top$.
Thus, DyRA reduces to the static rank-$(k+r)$ truncated SVD in the isotropic case.
\end{proof}

\section{Experimental Details}
\label{sup:exp_detail}
In this section, we provide additional details on the experimental setup.
Unless otherwise specified, we use one ALS iteration for DyRA across all experiments.
We initialize $\mR^{(0)}$ with the top right singular vectors of the residual weight $\mW_{\mathrm{res}}=\mW-\widehat{\mW}$, following the subspace interpretation in Section~\ref{paragraph:subspace}.
This initialization is performed once offline before inference and does not introduce input-dependent overhead.

\paragraph{Image Classification}
For image classification, we evaluate DyRA on ImageNet-1K using a ViT-L/32 model\footnote{\url{https://huggingface.co/google/vit-large-patch32-384}}~\cite{vit}~ pretrained on ImageNet-21K~\cite{imagenet21k}.
We apply DyRA to all linear layers in the model, including the MLP layers and the query, key, value, and output projections in the attention blocks.
We fix the rank of the dynamic residual approximation to $r=8$ and control the FLOPs by varying the configuration of the static weight approximation.

For low-rank approximation, we vary the static rank over $\{100,150,200,250,300,350,400,450\}$.
For Monarch, we vary the Monarch block size over $\{3,4,6,8\}$.
For BLAST, we fix the block partition to $4$ and vary the static rank over $\{100,150,200,250,300,350,400,450\}$.
For each structured approximation, we construct the corresponding weight-only baseline by sweeping its configuration to match the FLOPs of DyRA.

\paragraph{Dense Prediction}
For dense prediction, we evaluate DyRA on semantic segmentation and depth estimation using a ViT-7B-based DINOv3 model~\cite{dinov3}.
We evaluate semantic segmentation on ADE20K~\cite{ade20k} and depth estimation on NYUv2~\cite{nyuv2}.
For both tasks, we follow the image preprocessing and decoding protocol of DINOv3~\footnote{\url{https://github.com/facebookresearch/dinov3}}.
We apply DyRA uniformly to all linear layers in the model.

We fix the rank of the dynamic residual approximation to $r=48$ and vary the configuration of the static weight approximation to control the FLOPs reduction.
For depth estimation, we sweep the low-rank approximation rank over $\{436,750,1000,1250,1500,1750\}$, the BLAST rank over $\{433,746,995,1244,1492,1741\}$, and the Monarch block size over $\{2,3,4,6,8\}$.
For semantic segmentation, we sweep the low-rank approximation rank over $\{750,1000,1250,1500,1750\}$, the BLAST rank over $\{746,995,1244,1492,1741\}$, and the Monarch block size over $\{2,3,4,6,8\}$.
The corresponding weight-only baselines are obtained by adjusting the same configuration parameters to match the FLOPs of DyRA.

\paragraph{Speech Recognition}
For speech recognition, we evaluate DyRA on Wav2Vec2-Large\footnote{\url{https://huggingface.co/facebook/wav2vec2-large-960h-lv60-self}}~\cite{wav2vec2}, pretrained and fine-tuned on 960 hours of Libri-Light~\cite{libri_light} and LibriSpeech~\cite{libri}.
We report word error rate (WER) on the LibriSpeech test set, following the standard audio preprocessing and CTC-based decoding protocol used for Wav2Vec2.
DyRA is applied to the query and key projections in the attention blocks and to the two fully connected layers in each MLP block.
For both DyRA and the corresponding weight-only baselines, rank parameters are selected by grid search over a predefined set of configurations under target FLOPs reduction ratios of 25\% and 50\%.

We summarize each DyRA configuration as
$(r_{\mathrm{MLP}}, r_{\mathrm{attn}}, r_{\mathrm{res}})$,
where $r_{\mathrm{MLP}}$, $r_{\mathrm{attn}}$, and $r_{\mathrm{res}}$ denote the MLP rank, attention rank, and residual rank, respectively.
The selected configurations are $(370,210,48)$ and $(200,60,36)$ for BLAST at 25\% and 50\% FLOPs reduction, respectively, and $(440,160,44)$ and $(200,90,36)$ for low-rank at 25\% and 50\% FLOPs reduction, respectively.

\paragraph{Language Modeling}
For language modeling, we evaluate DyRA on LLaDA-8B\footnote{\url{https://huggingface.co/GSAI-ML/LLaDA-8B-Base}}~\cite{llada}, a large language diffusion model.
We consider two benchmarks: HumanEval~\cite{humaneval} for code generation and GSM8K~\cite{gsm8k} for mathematical reasoning.
We apply low-rank-based DyRA to the query and key projections in the attention blocks and to the MLP layers.
Following prior observations that query and key projections in LLMs are more amenable to compression~\cite{welore}, we use more aggressive FLOPs reduction for the query and key projections than for the MLP layers.

For HumanEval, we use two DyRA configurations targeting 20\% and 40\% total FLOPs reduction.
At 20\% reduction, we set the query/key static rank to 840 with residual rank 40, and the MLP static rank to 1000 with residual rank 150.
At 40\% reduction, we set the query/key static rank to 300 with residual rank 50, and the MLP static rank to 300 with residual rank 150.

For GSM8K, we again use two configurations targeting 20\% and 40\% total FLOPs reduction.
At 20\% reduction, we set the query/key static rank to 512 with residual rank 240, and the MLP static rank to 1024 with residual rank 380.
At 40\% reduction, we set the query/key static rank to 320 with residual rank 180, and the MLP static rank to 630 with residual rank 270.

\paragraph{Fine-tuning ViT on ImageNet}
For the fine-tuning experiment, we use a ViT-L/32 model pretrained on ImageNet-21K and fine-tune it on ImageNet-1K. We apply DyRA uniformly to all linear layers, including the MLP layers and all attention projections. Each model is fine-tuned for 15 epochs using the AdamW optimizer~\cite{adamw} with a learning rate of $2\times 10^{-5}$ and a batch size of 1024.

We fix the dynamic residual approximation rank to $r=16$ and vary the static weight rank over $\{96,128,160,196\}$, training a separate DyRA model for each static rank.
For each DyRA configuration, we construct a FLOPs-matched static low-rank baseline by adjusting its rank and fine-tune the corresponding baseline model using the same training setup.

\paragraph{Comparison with Input-Aware Decomposition.}
For comparison with prior input-aware decomposition methods, we use the approaches proposed in~\cite{asvd,svdllm}.
In Eq.~\ref{eq:aware_svd}, we instantiate $\mS$ as the Cholesky factor of the input covariance matrix computed from calibration activations.
Since input-aware decompositions can be sensitive to the choice of calibration data, we use the corresponding training split of each evaluation dataset to provide a strong baseline setting.
Specifically, we use the ImageNet training set for ImageNet experiments and the GSM8K training set for GSM8K evaluation.
Because HumanEval consists only of evaluation problems without a dedicated training split, we use the GSM8K training set as the calibration set for HumanEval experiments.
For each target FLOPs reduction ratio, the ranks are kept identical to those of the corresponding DyRA configurations used in the language modeling experiments, except that the static low-rank factors are initialized using input-aware decomposition.

\section{Additional Experiment Results}

\begin{table}[t]
    \centering
    \small
    \caption{
        Perplexity comparison on language modeling benchmarks without fine-tuning.
        Lower is better.
        Bold indicates the best performance under the same compression ratio.
    }
    \label{tab:ppl_results}
    \vspace{1em}
    \resizebox{0.65\linewidth}{!}{
    \begin{tabular}{c|l|ccc}
        \toprule
        \textbf{CR}
            & \textbf{Method}
            & \textbf{WikiText}
            & \textbf{PG-19}
            & \textbf{C4} \\
        \midrule
        0\%
            & Dense
            & \textbf{9.38}
            & \textbf{10.07}
            & \textbf{7.01} \\
        \midrule
        \multirow{3}{*}{20\%}
        & Low-Rank
            & 16.00(+6.62)
            & 15.93(+5.86)
            & 10.81(+3.80) \\
        & BLAST
            & 12.94(+3.56)
            & 13.05(+2.98)
            & 9.26(+2.25) \\
        & Low-Rank + DyRA
            & \textbf{12.03(+2.65)}
            & \textbf{12.48(+2.41)}
            & \textbf{8.75(+1.74)} \\
        \bottomrule
    \end{tabular}
    }
    \vspace{-1em}
\end{table}

\begin{table}[t]
    \centering
    \scriptsize
    \caption{
    Performance comparison across multiple-choice reasoning benchmarks.
    FR denotes the FLOP reduction ratio.
    Values in parentheses indicate accuracy changes relative to the dense model.
    Bold indicates the best performance under the same compression ratio.
    }
    \label{tab:llama_common}
    \vspace{1em}
    \setlength{\tabcolsep}{4pt}
    \begin{tabular}{c|l|cccccc|c}
        \toprule
        \textbf{FR}
            & \textbf{Method}
            & \textbf{PIQA}
            & \textbf{Wino}
            & \textbf{BoolQ}
            & \textbf{OBQA}
            & \textbf{ARC-e}
            & \textbf{ARC-c}
            & \textbf{Avg} \\
        \midrule
        0\%
            & Dense
            & \textbf{81.66}
            & \textbf{74.90}
            & \textbf{77.03}
            & 47.60
            & \textbf{80.18}
            & \textbf{50.85}
            & \textbf{68.70} \\
        \midrule
        \multirow{3}{*}{20\%}
        & Low-Rank
            & 76.66(-5.00)
            & 69.38(-5.52)
            & 69.42(-7.61)
            & 45.20(-2.40)
            & 73.53(-6.65)
            & 41.72(-9.13)
            & 62.65(-6.05) \\
        & BLAST
            & \textbf{78.56(-3.10)}
            & 67.64(-7.26)
            & 66.91(-10.12)
            & 44.80(-2.80)
            & 75.42(-4.76)
            & 43.43(-7.42)
            & 62.79(-5.91) \\
        & \textbf{Low-Rank }+ \textbf{DyRA}
            & 77.69(-3.97)
            & \textbf{69.93(-4.97)}
            & \textbf{71.50(-5.53)}
            & \textbf{48.60(+1.00)}
            & \textbf{76.05(-4.13)}
            & \textbf{44.11(-6.74)}
            & \textbf{64.65(-4.05)} \\
        \bottomrule
    \end{tabular}
\end{table}
\par

\subsection{Autoregressive LLM Prefilling}
\label{sup:llama}
In addition to the diffusion language model setting, we examine whether DyRA can also be applied to autoregressive LLMs, particularly in the prefilling regime where computation is dominated by large matrix multiplications over the input sequence.
Rather than measuring decoding throughput, we evaluate this setting through language modeling perplexity and multiple-choice reasoning benchmarks~\cite{piqa, wino,boolq,openbookqa,arc}, where each prompt and candidate continuation is processed in a forward pass.
To better reflect non-trivial prefilling workloads, we use few-shot evaluation prompts that provide extended input contexts for multiple-choice reasoning benchmarks.

We apply low-rank-based DyRA to the linear layers of LLaMA-7B~\cite{llama}, targeting a 20\% reduction in matrix multiplication FLOPs.
As shown in Tables~\ref{tab:ppl_results} and~\ref{tab:llama_common}, DyRA consistently improves over the corresponding low-rank approximation and also outperforms the BLAST baseline under the same FLOPs reduction.
These results suggest that DyRA can be effectively applied to the prefilling stage of autoregressive LLMs.

\begin{figure}[t]
    \begin{center}
    \includegraphics[width=1.0\linewidth]{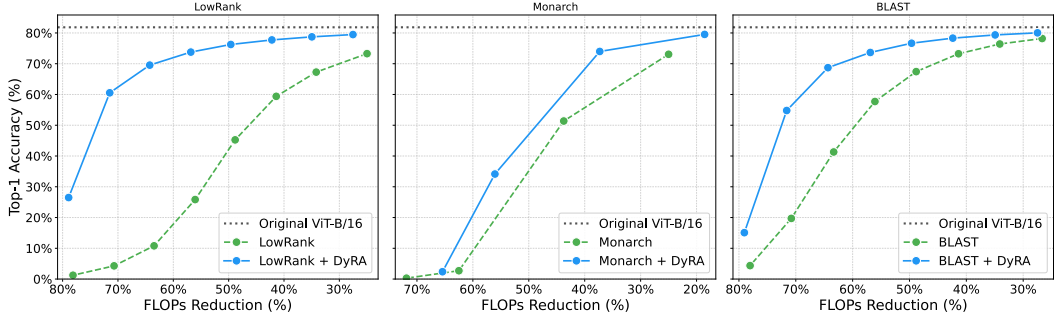}
    \end{center}
    \caption{Top-1 accuracy versus FLOPs reduction on ViT-B/16.
    Across low-rank, Monarch, and BLAST approximations, combining DyRA consistently improves the accuracy--efficiency trade-off over static weight approximation baselines.}
    \vspace{-1em}
    \label{fig:vitb}
\end{figure}

\subsection{Image Classification with ViT-B}
\label{sup:vitb}

To further examine whether DyRA remains effective across different model scales, we additionally evaluate image classification on ViT-B/16~\footnote{\url{https://huggingface.co/google/vit-base-patch16-224}}. 
Using the same experimental setup as in Section~\ref{sup:exp_detail}, we combine DyRA with static weight approximations and compare it against FLOPs-matched static baselines. 
As shown in Figure~\ref{fig:vitb}, DyRA consistently improves the accuracy--efficiency trade-off over the corresponding static weight approximations. 
These results demonstrate that DyRA's gains are not specific to the larger ViT-L model used in the main paper, but also generalize to a smaller ViT backbone.

\subsection{Convergence of ALS Residual Approximation}
\label{sup:convergence}
We provide additional details and results for the ALS convergence experiment discussed in Section~\ref{paragraph:convergence}.
We randomly sample 4,096 images from the ImageNet validation set and use them to extract the input activations $\mX$ at each linear layer of ViT-L.
For each layer, we construct the residual weight
$\mW_{\mathrm{res}}=\mW-\widehat{\mW}$,
where $\widehat{\mW}$ is obtained via a rank-196 truncated SVD of $\mW$.
We then apply ALS to approximate the residual product $\mX\mW_{\mathrm{res}}$ with target rank $r=16$.

At each ALS iteration $t$, we compute the relative Frobenius error between the residual product and its low-rank approximation:
$\frac{\|\mX\mW_{\mathrm{res}}-\mL^{(t)}\mR^{(t)}\|_F}{\|\mX\mW_{\mathrm{res}}\|_F}$.

We compare this error with the optimal rank-$r$ error obtained by applying SVD directly to $\mX\mW_{\mathrm{res}}$.

Figure~\ref{fig:als_converge_full} reports the convergence results averaged over the sampled images across all ViT linear layers. ALS rapidly approaches the optimal rank-$r$ error in nearly all layers, typically within two iterations.

\begin{figure}[t]
    \centering
    \begin{subfigure}{1.0\linewidth}
        \centering
        \includegraphics[width=1.0\linewidth]{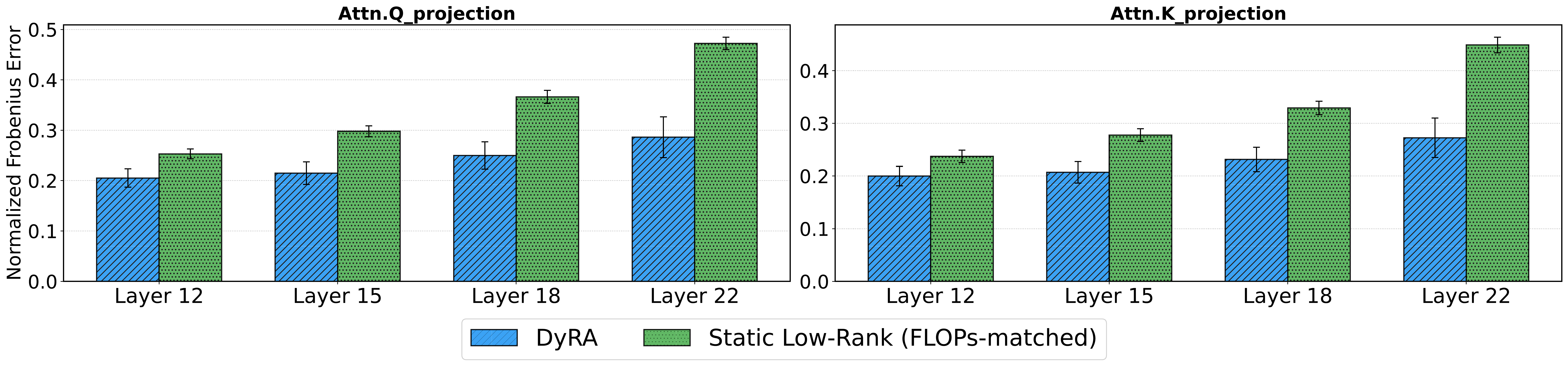}
        \caption{Attention query and key projections.}
        \label{fig:flops_matched_error_qk}
    \end{subfigure}

    \begin{subfigure}{1.0\linewidth}
        \centering
        \includegraphics[width=1.0\linewidth]{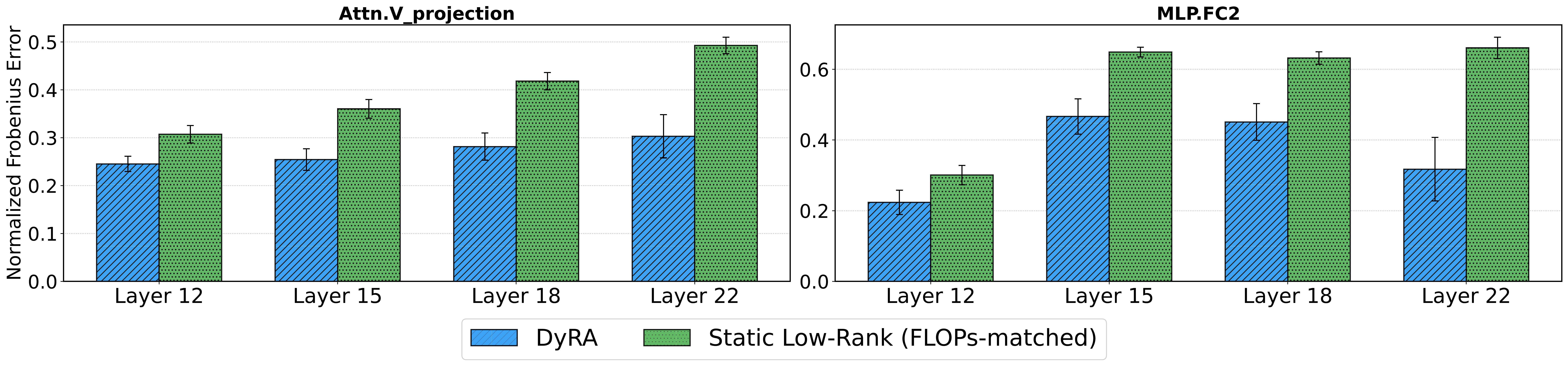}
        \caption{Attention value projection and second MLP linear layer.}
        \label{fig:flops_matched_error_vfc2}
    \end{subfigure}
    \caption{
    Additional FLOPs-matched output approximation error results on ViT-L sublayers omitted from the main paper. DyRA consistently achieves lower normalized Frobenius output error than the static low-rank baseline.
    }
    \vspace{-1.5em}
    \label{fig:flops_matched_error_full}
\end{figure}

\subsection{FLOPs-Matched Output Approximation Error}
\label{sup:flops_matched_error}

We provide additional details and results for the FLOPs-matched output error comparison discussed in Section~\ref{paragraph:error_comparison}.
Using the same ViT-L and ImageNet validation setup as in Section~\ref{sup:convergence}, we collect input activations $\mX$ from 4,096 validation images for each target linear layer.
For each layer, we compute a rank-196 truncated SVD approximation $\widehat{\mW}$ and apply a rank-16 DyRA correction to the residual product $\mX(\mW-\widehat{\mW})$ using one ALS iteration.

For the FLOPs-matched static baseline, we increase the low-rank approximation rank so that its exact FLOPs match the total FLOPs of DyRA, including the dynamic correction overhead.
We then compare the remaining normalized Frobenius output error, $\|\mX\mW-\widehat{\mY}\|_F / \|\mX\mW\|_F$, where $\widehat{\mY}$ denotes the approximated output.

Figure~\ref{fig:flops_matched_error_full} reports additional results for the remaining sublayers not shown in the main text: the query, key, and value projections in attention, and the second MLP linear layer.
Across these additional sublayers, DyRA consistently yields lower output approximation error than the FLOPs-matched higher-rank static low-rank baseline.

\section{Broader Impacts}
\label{sup:impact}
DyRA is a test-time inference acceleration method designed to reduce the computational cost of neural network inference. Its potential positive societal impact lies in improving the resource efficiency of AI systems, enabling model deployment under limited computational budgets and reducing the energy costs associated with inference. Since DyRA does not introduce a new application domain or change the intended use of the underlying model beyond approximation, we do not identify negative societal impacts specific to the proposed method. However, by making inference more efficient, DyRA may facilitate broader deployment of existing AI systems, and therefore may inherit the societal risks associated with those systems and their applications.

\clearpage
\begin{figure}[t]
    \centering
    \begin{subfigure}{1.0\linewidth}
        \centering
        \includegraphics[width=1.0\linewidth]{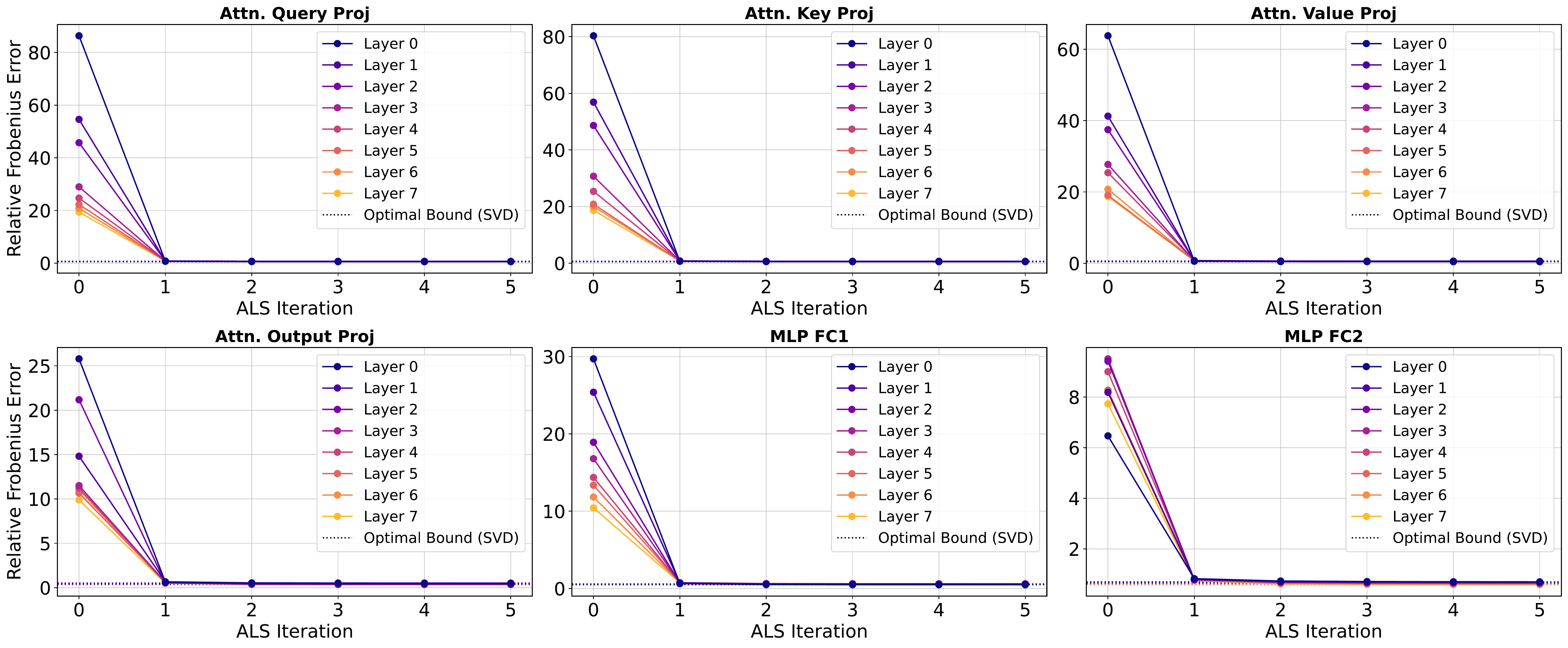}
        \caption{Layers 0--7 of ViT-L.}
        \label{fig:als_sup1}
    \end{subfigure}

    \begin{subfigure}{1.0\linewidth}
        \centering
        \includegraphics[width=1.0\linewidth]{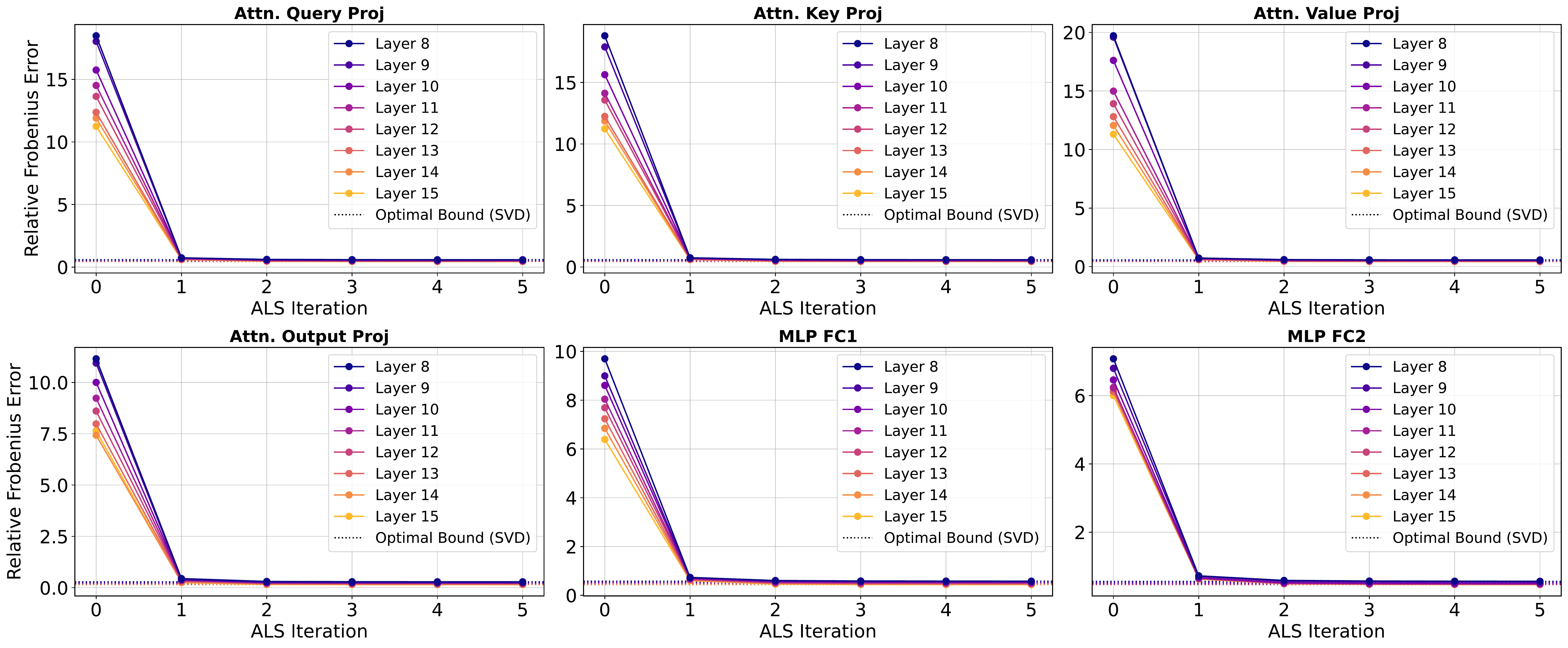}
        \caption{Layers 8--15 of ViT-L.}
        \label{fig:als_sup2}
    \end{subfigure}

    \begin{subfigure}{1.0\linewidth}
        \centering
        \includegraphics[width=1.0\linewidth]{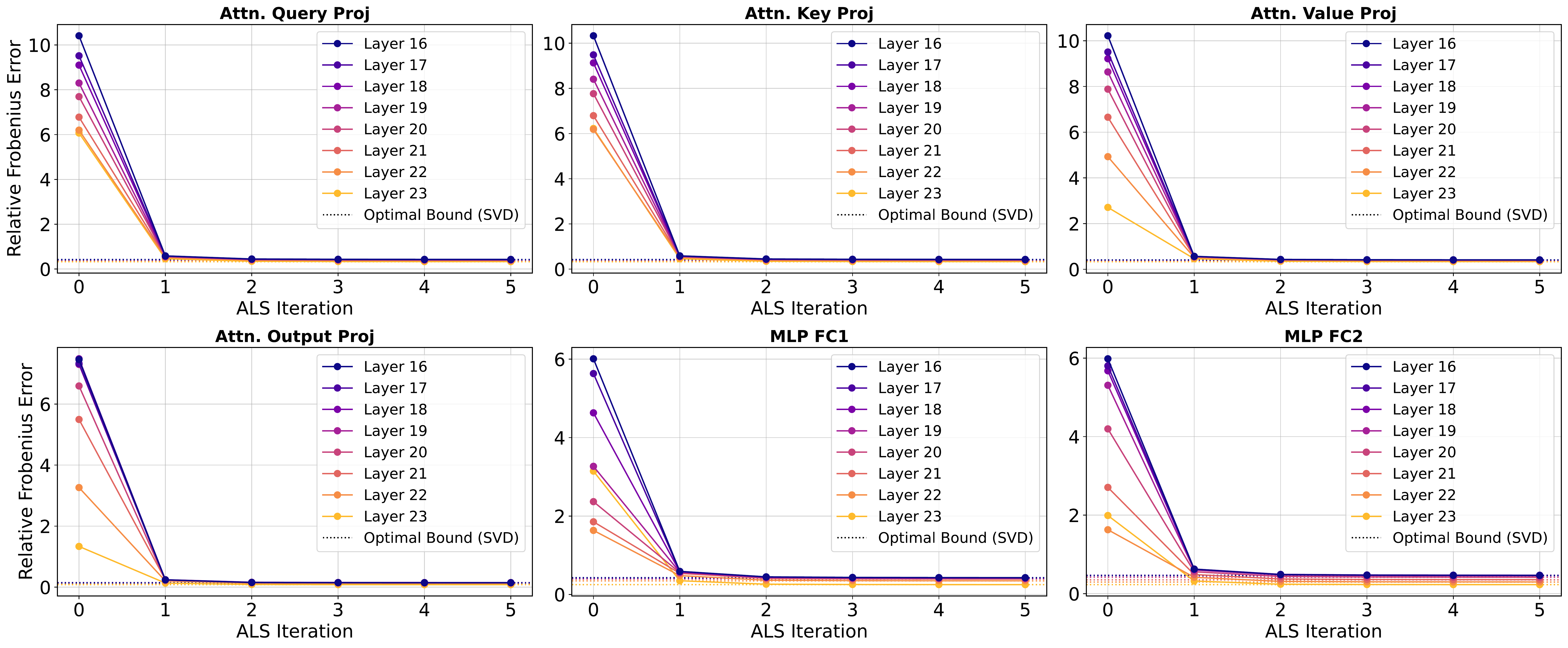}
        \caption{Layers 16--23 of ViT-L.}
        \label{fig:als_sup3}
    \end{subfigure}

    \caption{Convergence of ALS-based residual product approximation across ViT-L linear layers.}
    \vspace{-1.5em}
    \label{fig:als_converge_full}
\end{figure}

\end{document}